\documentclass[journal]{IEEEtran}
\ifCLASSINFOpdf
\else
\fi

\usepackage{epsfig}
\usepackage{amsmath}
\usepackage{hhline}
\usepackage{multirow}
\usepackage{eucal}
\usepackage{cite}
\usepackage{amssymb}
\usepackage{subfigure}
\usepackage{color}
\usepackage{amsthm}
\usepackage{enumerate}
\usepackage{bm}
\newtheorem{thm}{Theorem}
\newtheorem{deff}{Definition}
\newtheorem{lem}{Lemma}

\newtheorem{rem}{Remark}

\newcommand*{\QEDA}{\hfill\ensuremath{\blacksquare}}%

\newcommand{\bM}{\bm{M}}
\newcommand{\bC}{\bm{C}}
\newcommand{\bg}{\bm{g}}
\newcommand{\bq}{\bm{q}}
\newcommand{\bp}{\bm{p}}
\newcommand{\be}{\bm{e}}
\newcommand{\bx}{\bm{x}}
\newcommand{\btau}{\bm{\tau}}
\newcommand{\by}{\bm{y}}
\newcommand{\bD}{\bm{D}}
\newcommand{\bK}{\bm{K}}
\newcommand{\bJ}{\bm{J}}
\newcommand{\bI}{\bm{I}}
\newcommand{\bLambda}{\bm{\Lambda}}
\newcommand{\bzeta}{\bm{\zeta}}
\newcommand{\bGamma}{\bm{\Gamma}}
\newcommand{\bff}{\bm{f}}
\newcommand{\bw}{\bm{w}}
\newcommand{\bA}{\bm{A}}
\newcommand{\bB}{\bm{B}}
\newcommand{\bzero}{\bm{0}}
\newcommand{\bP}{\bm{P}}
\newcommand{\bQ}{\bm{Q}}

\begin{document}
%
\title{A Passivity-based Nonlinear Admittance Control\\with Application to Powered Upper-limb Control\\under Unknown Environmental Interactions}

%
%
%

\author{Min Jun Kim${}^{a,b}$, Woongyong Lee${}^{a}$,  Jae Yeon Choi${}^{c}$,  Goobong Chung${}^{c}$, Kyung-Lyong Han${}^{d,e}$, Il Seop Choi${}^{e}$, Christian Ott${}^{b}$, and Wan Kyun Chung${}^{a}$,~\IEEEmembership{Fellow,~IEEE}
\thanks{This work was supported in part by the convergence technology development program for bionic from the National Research Foundation of Korea (No. 2014M3C1B2048632), and by Technical Research Laboratories, POSCO.}
\thanks{${}^{a}$Robotics Laboratory, School of Mechanical Engineering, Pohang University of Science and Technology (POSTECH), Pohang, 37673, Gyung-buk, Korea 
{\tt\small \{mjkim0229, leewoongyong, wkchung\}@postech.ac.kr}}%
\thanks{${}^{b}$Robotics and Mechatronics Center, German Aerospace Center (DLR), Wessling, Germany 
{\tt\small \{minjun.kim, christian.ott\}@dlr.de}}%
\thanks{${}^{c}$Korea Institute of Robot and Convergence (KIRO), Pohang, 37666, Gyung-buk, Korea 
{\tt\small \{bluedus, goobongc\}@kiro.re.kr}}%
\thanks{${}^{d}$Artificial Intelligence Team, Samsung Electronics, Suwon-si, 16677, Gyeonggi-do, Korea
	{\tt\small kl.han@samsung.com}}%
\thanks{${}^{e}$POSCO Technical Research Laboratories, Pohang, 37859, Gyung-buk, Korea
{\tt\small \{kevinhan75, i.s.choi\}@posco.com}}%
}

\maketitle

\begin{abstract}
This paper presents an admittance controller based on the passivity theory for a powered upper-limb exoskeleton robot which is governed by the nonlinear equation of motion. Passivity allows us to include a human operator and environmental interaction in the control loop. The robot interacts with the human operator via F/T sensor and interacts with the environment mainly via end-effectors. Although the environmental interaction cannot be detected by any sensors (hence unknown), passivity allows us to have natural interaction. An analysis shows that the behavior of the actual system mimics that of a nominal model as the control gain goes to infinity, which implies that the proposed approach is an admittance controller. However, because the control gain cannot grow infinitely in practice, the performance limitation according to the achievable control gain is also analyzed. The result of this analysis indicates that the performance in the sense of infinite norm increases linearly with the control gain. In the experiments, the proposed properties were verified using 1 degree-of-freedom testbench, and an actual powered upper-limb exoskeleton was used to lift and maneuver the unknown payload.
\end{abstract}


%
\IEEEpeerreviewmaketitle

\section{Introduction}
\label{sec:intro}
\IEEEPARstart{I}{n} recent decades, various types of exoskeletons (or wearable robots) have been developed for various purposes \cite{guizzo2005rise, marcheschi2011body, noda2014development, hsieh2015mechanical ,kuan2018high}. Among them, carrying heavy payloads is one of the most relevant applications in the industries. Fig. \ref{fig:hardware} shows a powered upper-limb exoskeleton robot developed to carry arbitrary heavy payloads.

\begin{figure}
\centering
\subfigure[]
{\includegraphics[scale=0.31]{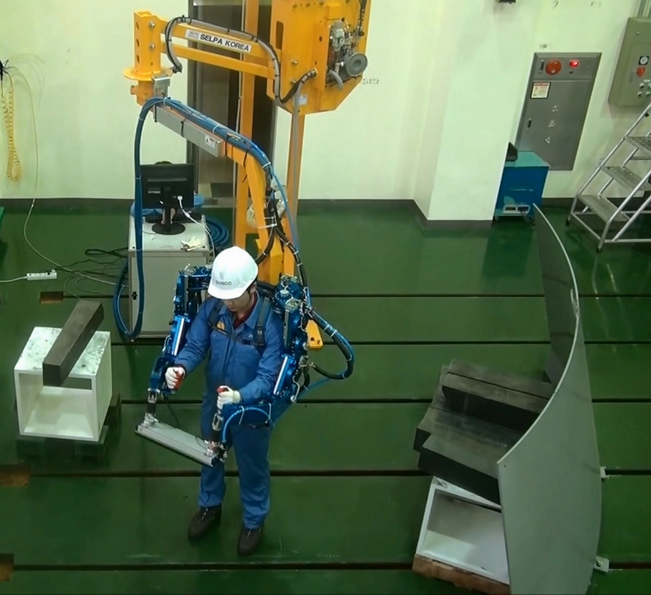}}
\centering
\subfigure[]
{\includegraphics[scale=0.28]{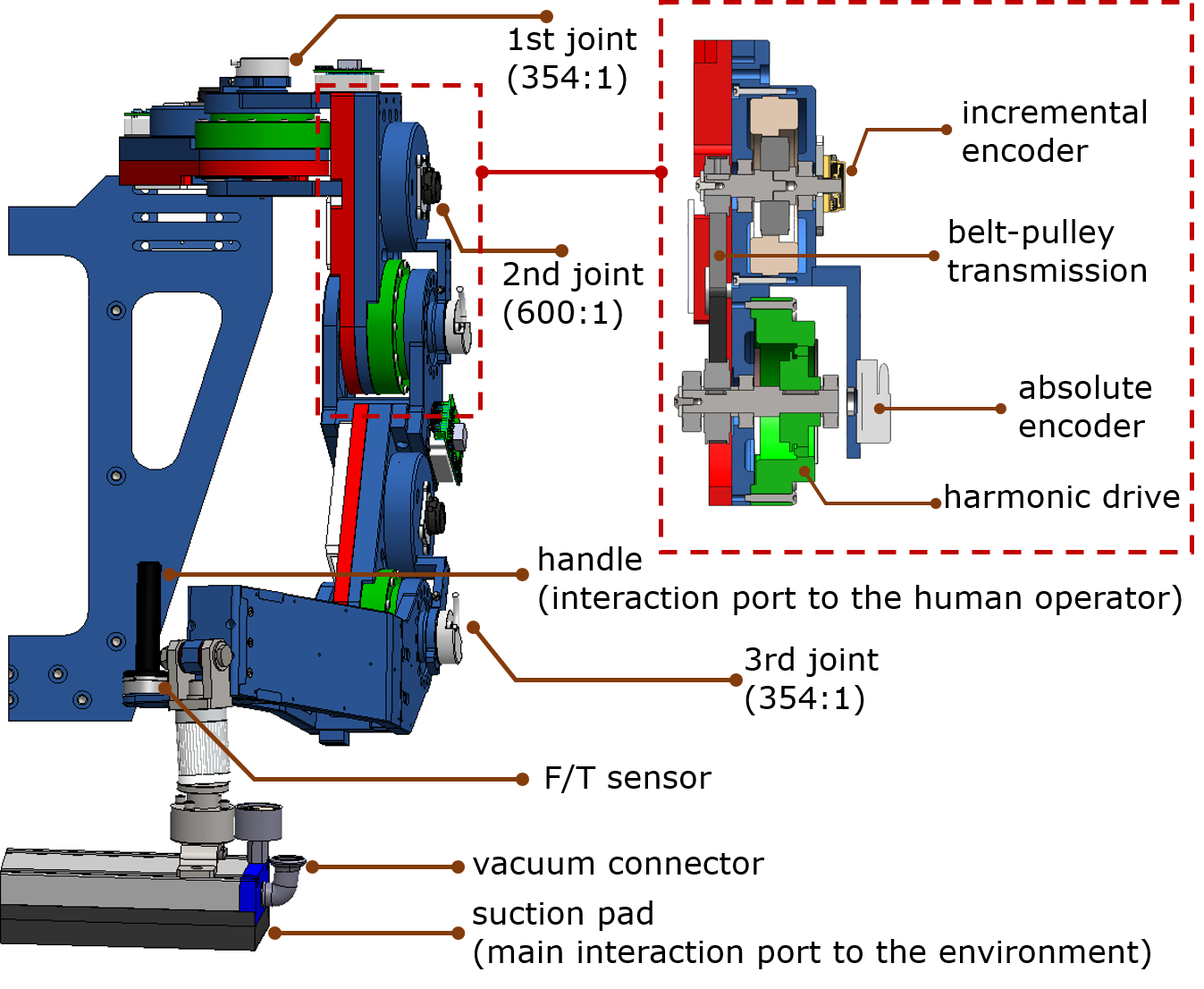}}
\caption{Powered upper-limb exoskeleton robot for carrying heavy payloads. (a) A mock-up of the workplace where the robot will be used. (b) The descriptions of left arm. The highlight (red dashed outline) shows details of the second joint; the other joints are similar.}
\label{fig:hardware}
\end{figure}

The main scenario considered in this study, which is typical of industrial applications, involves the operator approaching and lifting a payload of unknown mass (of up to 50 $\mathrm{kg}$) and then moving it to another location.\footnote{ \label{footnote1}
To accommodate a larger workspace, the actual system includes additional elements such as horizontal passive joints and a waist joint. These elements were neglected in the present study because they are irrelevant to the main theme of the paper.} 
A pneumatic suction pad on the end-effector adheres to the payload, where an air pump is used to create a vacuum. The mass of the payload, which is not known, varies from 10 to 50 $\mathrm{kg}$ (i.e., up to 25 $\mathrm{kg}$ for each arm). During the operation, the robot interacts with the environment as well as the human operator. The interaction with a human operator can be measured by a F/T sensor attached to the handle. Interactions with the environment, which cannot be detected by any sensor, mainly occur when the suction pad attaches to or releases the payload.

There are two major difficulties in accomplishing the scenario. The first one is that the robot dynamics is significantly affected by the inertia of the payload, which is not known. Hence, the robot dynamics should be treated to be unknown. The second one is the absence of the force sensor that measures (or detects) the interaction between the robot and environment. Recall that, in our scenario, environmental interaction occurs at the suction pad when attaching to or releasing a payload. Due to the lack of information, realizing natural interaction with the environment becomes nontrivial.

In this paper, these difficulties are alleviated using a passivity-based nonlinear admittance controller of which control structure is constructed by feedback interconnection of passive subsystems. Passivity is employed to account human  and environmental interaction in the control loop. Nonlinear admittance control scheme is employed to capture model uncertainty arises from uncertain nonlinear robot dynamics. It should be noted that stability is guaranteed by passive structure of the controller without knowing the system model. Therefore, unknown mass does not impact stability.

This paper is organized as follows. In Section \ref{sec:problem_statement_and_related}, the problem formulation is stated and related works are reviewed. After introducing the standard admittance controller in Section \ref{sec:standard_adm}, the proposed control scheme is presented in Section \ref{sec:passivity_based_admittance_controller} with theoretical justification in Section \ref{sec:theory}. In Section \ref{sec:exp_1dof}, the proposed method is verified using a 1 Degree of Freedom (DoF) testbed. In Section \ref{sec:powered}, the experimental results with the powered upper-limb exoskeleton are presented. Section \ref{sec:conclusion} presents the conclusions of this study. This study is the extension of the works presented in \cite{kim2016powered, kim2016passivity}. Compared to them, this paper presents a more complete form of the control structure which is more appealing in the passivity point of view, and introduces additional theorem (Theorem \ref{thm:passivity} in this paper) that states input-output (I/O) stability.

\section{Problem Statement and Related Work}
\label{sec:problem_statement_and_related}

\subsection{Problem Statement (controller specification)}

When the ideal friction-free robot is supporting a light payload, joint damping plus gravity compensation control can be used to allow the human operator to easily operate the robot. Assuming passive environment, as far as a human operator injects the finite energy, the controlled robot will dissipate the energy obtained from the human operator.

In practice, however, to achieve high torque capability, the robot is actuated with electric motors with high-ratio reduction drives (354:1 for the first and third joints, and 600:1 for the second joint; see Fig. \ref{fig:hardware}b).\footnote{
	In fact, the robot model is almost linear w.r.t. joint angles due to the high reduction ratio. However, this paper employs nonlinear techniques as the task space dynamics (which is of our final interest) is still highly nonlinear \cite{khatib1987unified}, as will be presented in Section \ref{sec:task_space_control}. 
} The reduction drives not only amplify reflected inertia, but also introduce significant joint frictions (in our setup, greater than 15 $\mathrm{Nm}$ is required in each joint to overcome the stiction). Consequently, it becomes difficult to operate the robot for human operators. The situation becomes worse when a heavy payload is being supported. The presence of the payload further increases the inertia (or mechanical impedance) of the robot. In addition, although the gravity compensation for the payload weight is essential, it is infeasible because the payload is unknown. Moreover, because of the computational, communication, and, most importantly, cost burdens (including potential maintenance cost), we limited the sensors used: in addition to the joint encoders, F/T sensors at the handles are available to measure the user interaction force. Due to the lack of sensory information, interaction with the environment (which occurs when attaching/detaching the payload) cannot be detected. 

To summarize, the control should satisfy the following requirements:
\begin{itemize}
\item{stability is maintained in the presence of human operator;}
\item{natural interaction is maintained during the unknown interactions (due to the sensorless setup) with a passive environment;}
\item{the actual robot model is not required while overcoming the effects of high-ratio reduction drives and unknown payloads.}
\end{itemize}

To match all these requirements, this paper proposes a passivity-based nonlinear admittance controller that makes the actual robot behave like a user-defined nominal model. Because the nominal model can be chosen freely, ideally speaking, we can make the system to have low inertia, and to be free of joint friction and gravity. Consequently, by setting zero stiffness for the admittance filter of the admittance controller, the robot can be easily operated by human input. Human operator and environment are incorporated in the control loop using the passivity theory.

\subsection{Related Work}
\label{sec:related_work}

The control problem can be simplified if rich sensory information is available. For example,  \cite{aghili2009impedance} proposed a robust impedance control scheme for carrying heavy payloads using wrist F/T sensors. If joint torque sensors are installed at every joint, dynamics caused by payload can be estimated using, for example, residual-based estimators \cite{de2003actuator, kim2015design, kim2015disturbance}. However, the use of additional force sensors increases the cost and reduces the robustness of the system in practice. For this reason, our only compromise was F/T sensors at the handles to read the the user interaction force. 

For the sake of completeness, we remark that there is another branch of works that try to extract human motion intention using bio sensors \cite{hochberg2012reach}, motion capture system \cite{lee2006fpga}, and customized force sensors \cite{huang2015control}. However, this paper does not try to combine these approaches to focus only on the control design. We also remark that there have been efforts to extract not only the human interaction force but also the environmental interaction force from a single F/T sensor at the handle \cite{park2012development}. However, this problem requires full information of the system model. For such a reason, the environmental interaction is assumed to be unknown in this paper.

Using only the F/T sensor at the handles, the problem can be formulated as a robust control problem that makes a robot obey the F/T sensor signals. In this sense, the admittance control approach is noteworthy; it makes the actual system follow a user-defined behavior (or a user-defined model).

 The standard admittance controller, however, has two inherent limitations directly related to the present application. First, stable interaction with human operator can be guaranteed with properly chosen parameters that makes the admittance function (transfer function from force to velocity) positive real. This methodology is often called natural admittance control \cite{dohring2003passivity}. However, because our robot is governed by the nonlinear equation of motion, transfer function analysis is not allowed. Second, the natural interaction behavior can be realized only on a known (i.e., predefined) interaction port at which force sensor is located. Unfortunately, the robot tends to stick to the environment if the interaction is not predefined. This phenomenon is called wall sticking effect \cite{kim2016passivity}. Again, note that the environmental interaction is not known to the robot in our scenario. \cite{kikuuwe2014sliding} pointed out a similar limitation, and proposed a position controller for the admittance control scheme that has an intermediate object to which robot converges with a predefined bounded trajectory. However, because the admittance filter of our setup has zero stiffness as addressed earlier (recall that we want the robot to move only in response to the human input), the trajectory of the admittance controller can grow unbounded when interacting with unknown environment.

To overcome the limitations, passivity-based design that can account for a human and environmental interactions in the control loop is employed in this paper. Passivity-based approaches are frequently used in analyzing human-machine-environment coupled systems (e.g., \cite{losey2016time, alghooneh2016passive, kim2018passive}) owing to a property that the feedback interconnection of passive subsystems preserves passivity. Several design methodologies based on the passivity property are proposed for rehabilitation robots \cite{ zhang2015passivity, hsieh2017design}. However, due to the lack of the model shaping (MS) property, they are not appropriate for the problem of interest. In this paper, a controller is regarded as an admittance controller if it satisfies the MS property defined as  follows.
	
\begin{deff}[MS property]
	\label{def:MS_property}
	If the closed-loop dynamics can be shaped into the user-defined model as the control gain goes to infinity, it is said that the controller has MS property. 
\end{deff}

This paper proposes a passivity-based nonlinear admittance controller with the following three properties. First, passive control structure reduces the wall-sticking effect under the unknown interaction between robot end-effector and environment. Second, the MS property can be shown using the two time scale analysis. Third, the performance limit when the infinite control gain is not achievable is analyzed, inspired by \cite{kim2015bringing} in which the performance analysis of nonlinear $\mathcal{H}_\infty$ optimal controllers was performed.

\begin{figure}
	\centering
	{\includegraphics[scale=0.45]{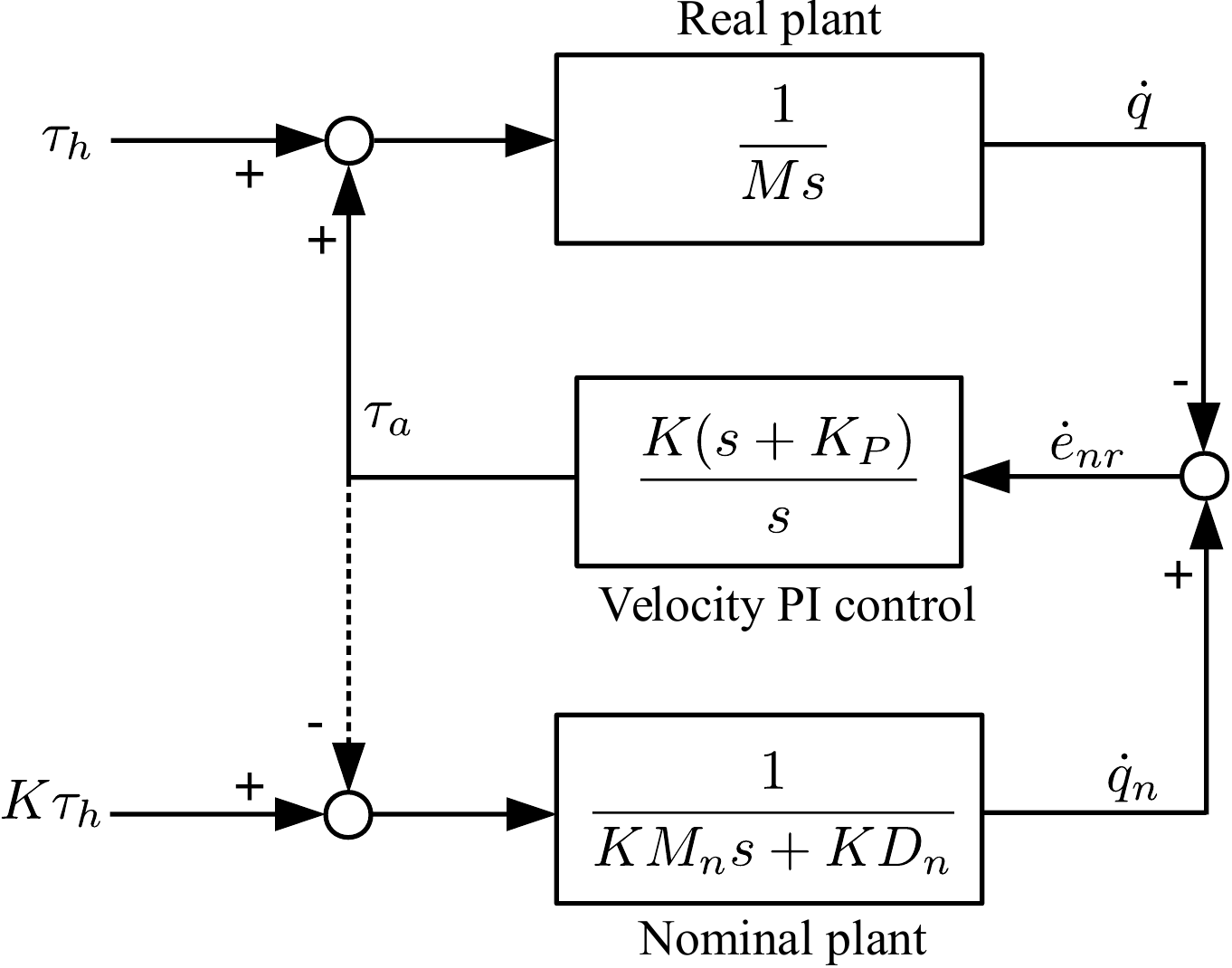}}
	\caption{Without dashed line: standard admittance controller. With dashed line: admittance control with passive structure. Note that the dashed line generates a bi-directional energy flow between the real and nominal plants.}
	\label{fig:adm_ctr_1dof}
\end{figure}

\section{Standard Admittance Controller}
\label{sec:standard_adm}

The structure of a standard admittance controller is described in Fig. \ref{fig:adm_ctr_1dof} (without dashed line). The admittance control aims at realizing desired behavior at a pre-defined interaction port which is usually the end-effector of a robot. Admittance controller renders the system dynamics into the user-defined nominal one $1/(M_n s+D_n)$, where $M_n$, $D_n$ represent the nominal mass and damping, respectively. Note that the nominal dynamics is defined in such a way that the mass can be manipulated by human operator (i.e., spring behavior is not included), which is the main problem tackling in this paper.

In Fig. \ref{fig:adm_ctr_1dof}, the nominal trajectory $\dot{q}_n$ is generated according to the user-defined nominal dynamics, and the system trajectory follows it by the admittance control input $\tau_a$. This paper considers a velocity PI controller\footnote{
	When implementing a classical admittance control, sometimes, $\dot{q}_n$/$q_n$ instead of $\tau_a$ is commanded using velocity/position servo. If the robot is not torque controlled, $\tau_a$ should be measured/estimated to realize the dashed line in Fig. \ref{fig:adm_ctr_1dof}.
}: 
\begin{align}
\tau_a = K (\dot{e}_{nr} + K_P e_{nr}),
\label{eq:adm_ctrl_1dof}
\end{align}
where $e_{nr}=q_n - q$, and $K_P, K>0$ are scalar control gains. The subscript `$n$' stands for `nominal' in the sense that the nominal (or desired) behavior is achieved if $q$ follows $q_n$. Here, notice that $e_{nr}(0)=0$ is implicitly assumed as $q_n$ can be always initialized to be identical to $q$.

In Fig. \ref{fig:adm_ctr_1dof}, the transfer function from $\tau_h$ to $\dot{q}$ is given by\footnote{$T_{(\cdot)}$ and $Q$ represent Laplace transform of $\tau_{(\cdot)}$ and $q$, respectively.}
\begin{align}
\frac{\dot{Q}}{T_H}=\frac{N(s) s + K s + K K_P}{M N(s) s^2 +  K N(s) s +  K K_P N(s)}
\label{eq:admittance_standard},
\end{align} 
where $N(s)=M_ns + D_n$. One interesting observation is that (\ref{eq:admittance_standard}) becomes 
\begin{align}
\frac{\dot{Q}}{T_H}= \frac{1}{N(s)}= \frac{1}{M_n s + D_n},
\end{align} 
as the control gain $K$ goes to infinity. This observation indicates that the MS property (Definition \ref{def:MS_property}) is satisfied, and therefore (\ref{eq:adm_ctrl_1dof}) is an admittance controller.

However, the standard admittance control achieves the user-defined dynamics only at the pre-defined admittance port. If an interaction occurs at undefined port, the admittance controller may not work properly, as shown in the following example.

{\it\underline{Motivating example}:}

Assume that the mass is interacting with an unknown wall as shown in Fig. \ref{fig:limitation_unknown_wall}. Although a human operator is applying force $\tau_h$, the mass cannot move forward due to the unknown wall ($\tau_{e}$ is also unknown to the controller). However, $q_n$ increases during the contact, meaning that the nominal trajectory drifts away from the real position. When the user applies force to the opposite direction, the mass cannot follow this intention because the controller should wait until $q_n$ to come back. This phenomenon is called wall-sticking effect in \cite{kim2016passivity}, and this effect hinders the system from  interacting with the unknown environment naturally.
\QEDA

\begin{figure}
	\centering
	\includegraphics[scale=0.38]{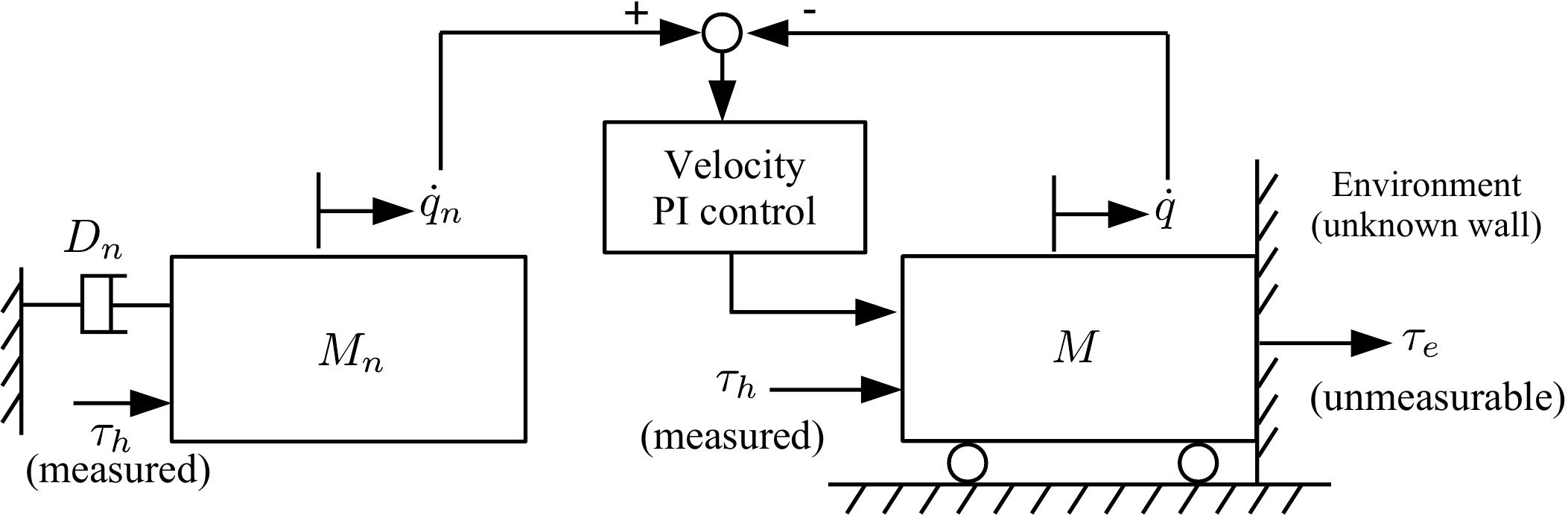}
	\caption{Standard admittance controller is applied under the unknown environmental interaction with an unexpected wall. Although the real mass cannot move ($\dot{q}=0$), the nominal velocity $\dot{q}_n$ (and also $q_n$) is keep increasing. This is, physically speaking, due to the lack of passivity in the control structure; there is no bi-directional energy flow between real and nominal masses.}
	\label{fig:limitation_unknown_wall}
\end{figure}

\section{Passivity-based Admittance Controller}
\label{sec:passivity_based_admittance_controller}

This section presents the admittance controller for the nonlinear robotic systems with passivity property. Since the passivity theory is the main tool in this paper, it is briefly reviewed in Section \ref{sec:preliminary}. After extending the standard admittance controller in the linear domain with passivity (Section \ref{sec:passive_adm_lin}), the nonlinear extension is presented (Section \ref{sec:passive_adm_proposed}).

\subsection{Preliminaries}
\label{sec:preliminary}

Consider a nonlinear system 
\begin{align}
\dot{\bx} = \bm{f}(\bx) + \bm{g}(\bx) \bm{u}, \;\; \by = \bm{h}(\bm{x}).
\end{align}
The I/O pair ($\bm{u} , \by$) is said to be output strictly passive if
\begin{align}
\int \bm{u}^T \by \mathrm{d}t \geq \alpha \int \by^T \by \mathrm{d}t
\label{eq:passivity}
\end{align}
is satisfied with some constant $\alpha >0$ (passive when $\alpha=0$). Output strictly passivity physically implies that the energy of system can only be dissipated.  Consequently, as long as the input $\bm{u}$ has the finite energy (i.e., $\bm{u} \in \mathcal{L}_2$),  the system will dissipate the injected energy eventually. Mathematically, this observation is called finite gain $\mathcal{L}_2$ stable:
\begin{align}
\int_0^T ||\bm{y}||_2^2 \mathrm{d}t \leq \gamma \int_0^T ||\bm{u}||_2^2 \mathrm{d}t + \delta,
\end{align}
where $\gamma,\delta>0$ are some constants. More rigorous mathematical treatment can be found in \cite{khalil2002nonlinear}.

For linear systems, it is relatively easy to investigate passivity of the system using transfer function analysis.\footnote{
	The I/O pair is passive if and only if the associated transfer function is so-called positive real.
} However, since this analysis is not applicable to nonlinear systems, this paper relies on the fact that the feedback interconnection of passive systems preserves passivity.

We remark that the I/O pair $(\dot{\be}_{nr}, \btau_a)$ of the velocity PI control (\ref{eq:adm_ctrl_1dof}) is output strictly passive because
\begin{align}
\nonumber \int \btau_a^T \dot{\be}_{nr} =& \int ( K K_p \be_{nr} + K \dot{\be}_{nr})^T \dot{\be}_{nr} \\
\geq& K \int \dot{\be}_{nr}^T  \dot{\be}_{nr}.
\label{eq:passivity_PI}
\end{align}
Here,  $\be_{nr}=0$ is assumed implicitly as addressed earlier. It is important to notice that the velocity PI control is equivalent to the position PD controller. However, the former one is used throughout the paper because the passivity theory utilizes the relation between velocity and force/torque. Another advantage of the passivity is that it provides clear physical understanding when applied to the mechanical system; the velocity PI control renders virtual spring-damper elements.

\subsection{Passivity-based admittance control for linear systems}
\label{sec:passive_adm_lin}

To alleviate the wall sticking effect, the standard admittance controller is extended using the passivity-based design. Based on the fact that the velocity PI control is a passive controller (recall (\ref{eq:passivity_PI})), it should be fed back to the nominal plant as well as the real plant to preserve passivity. Noting that the real and nominal plants are passive systems, the overall control structure is then constructed by feedback interconnections of passive subsystems, as shown in Fig. \ref{fig:adm_ctr_1dof} with dashed line.

Although the passivity can be shown from the structural point of view, it is not clear if this controller is an admittance controller or not. To investigate this, let us show the MS property by investigating the admittance transfer function:
\begin{align}
\nonumber &\frac{\dot{Q}}{T_H}=\\
& \frac{ N(s) s + (K+D_n)s + K K_P + K_P}{M N(s) s^2 + Ms^2+K N(s) s + M K_Ps + K K_P N(s)}.
\label{eq:admittance_proposed}
\end{align}
As $K \rightarrow \infty$, (\ref{eq:admittance_proposed}) becomes $1/N(s)$. Therefore, the MS property is satisfied, and this control scheme is an admittance controller.

Physically, noting that the velocity PI control is applied to both real and nominal plants with opposite signs, it allows bi-directional energy flow by spring-damper elements (Fig. \ref{fig:proposed_adm_ctr_concept}a). As $K \rightarrow \infty$, the spring-damper becomes rigid, so the controlled system behaves like Fig. \ref{fig:proposed_adm_ctr_concept}b. Consequently, the user-defined dynamics is achieved as the gain $K$ increases.

{\it\underline{Revisiting the motivating example}:}

Under the interaction with an unknown wall, the controlled mass can be described by Fig. \ref{fig:proposed_adm_ctr_with_wall}. $\dot{q}_n$ (and $q_n$) does not keep increasing because of the spring element. Mathematically speaking, when $\dot{q}=0$, $\dot{q}_n$ is governed by $M_n \ddot{q}_n +  (D_n+1) \dot{q}_n +  K_P (q_n - q_{e})= \tau_h$, 
where $q_{e}$ means the position at which the real mass contacts the environment (unknown wall). Hence, $\dot{q}_n$ converges to zero and $q_n$ will converge to a certain equilibrium point because of the spring $K_P$. Therefore, the proposed controller reduces the wall-sticking effect. \QEDA

\begin{figure}
	\centering
	\subfigure[]
	{\includegraphics[scale=0.32]{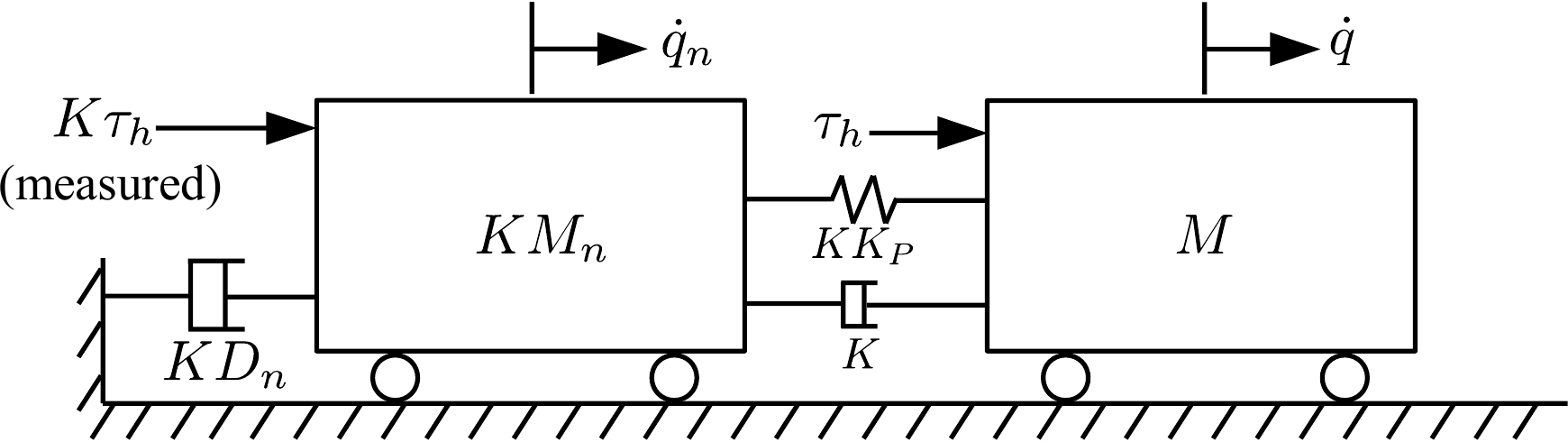}}
	\centering
	\subfigure[]
	{\includegraphics[scale=0.32]{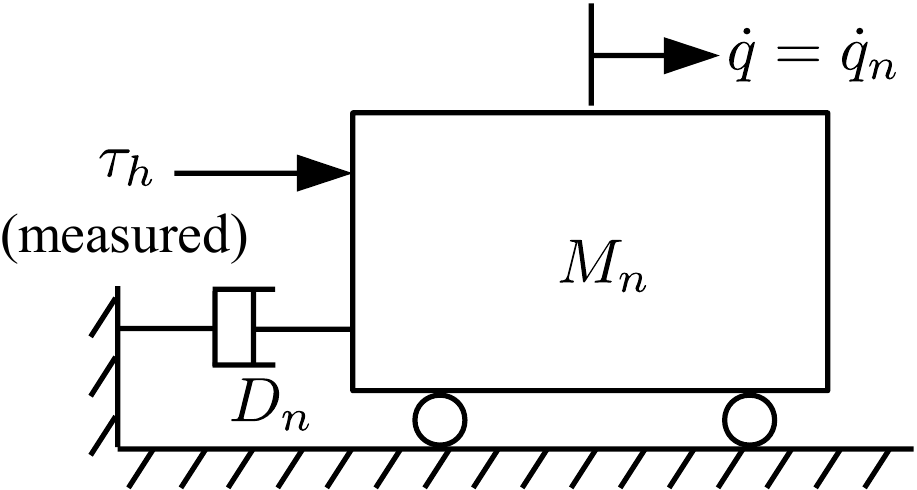}}
	\caption{(a) Conceptual diagram of the proposed passivity-based admittance controller. Recall that the velocity PI control renders virtual spring-damper elements. (b) As $K$ goes to infinity, spring-damper act as rigid connection, and thereby the controlled system behaves like the nominal plant.}
	\label{fig:proposed_adm_ctr_concept}
\end{figure}

\begin{figure}
	\centering
	\includegraphics[scale=0.37]{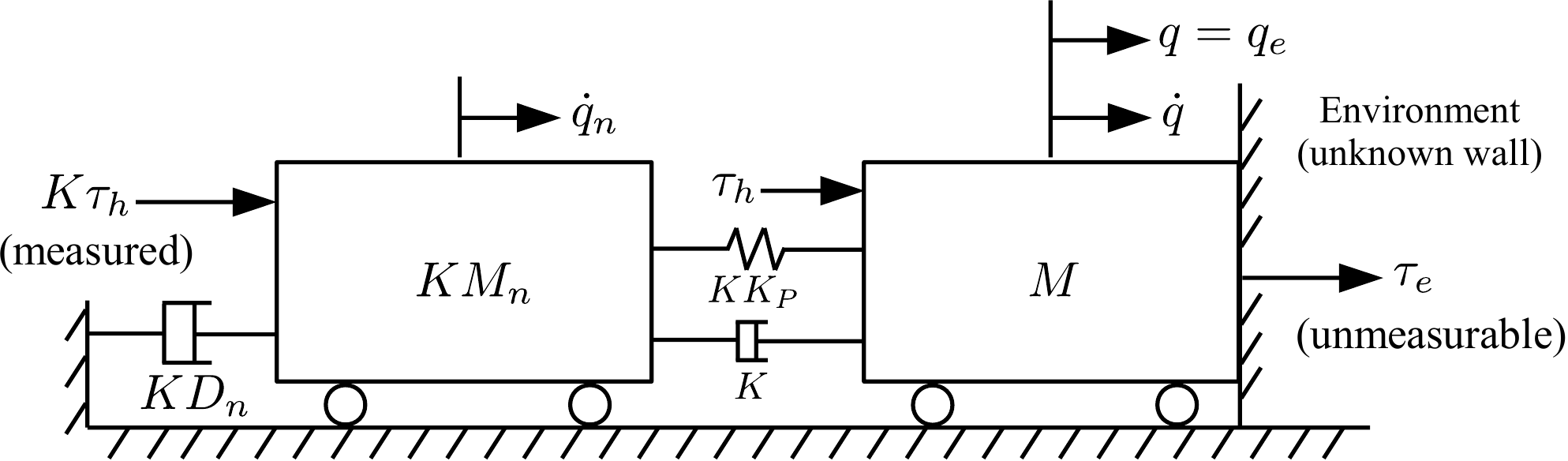}
	\caption{Proposed passivity-based admittance controller is applied under the unknown environmental interaction with an unexpected wall. Unlike the standard admittance control case, $\dot{q}_n$ as well as $\dot{q}$ does not keep increasing due to the feedback force of the spring element.}
	\label{fig:proposed_adm_ctr_with_wall}
\end{figure}

\subsection{Passivity-based admittance control for nonlinear systems}
\label{sec:passive_adm_proposed}

Although the passive admittance control is shown in the previous section, it is not applicable to nonlinear systems which cannot be represented by transfer functions. This section shows the nonlinear extension (Fig. \ref{fig:control_str}). The robot dynamics can be expressed as
\begin{align}
\bM(\bq)\ddot{\bq} + \bC(\bq,\dot{\bq})\dot{\bq} = \btau_a + \btau_h + \underbrace{\btau_f + \btau_e - \bg(\bq)}_{=\btau_d}
\label{eq:real_dyn}
\end{align}
using commonly used $\bM,\bC,\bg$, and $\bq$. Here, $\bC$ can be always selected such that $\dot{\bM}-2\bC$ is skew-symmetric. Note also that the actual $\bM$, $\bC$, $\bg$ are unknown and even may vary significantly because of the unknown payload. $\btau_a$ is the admittance input to be defined, $\btau_h$ is the interaction torque caused by the human operator, $\btau_f$ is the friction torque, and $\btau_e$ is the torque caused by environmental interactions. $\btau_f$, $\btau_e$, and $\bg$ are collectively represented by $\btau_d$ in the following.

In addition to the real robot dynamics, define the desired dynamics to achieve by
\begin{align}
\bM_n(\bq) \ddot{\bq} + (\bC_n(\bq,\dot{\bq}) + \bD_n) \dot{\bq} = \btau_h,
\label{eq:nom_dyn}
\end{align}
where $\dot{\bM}_n-2\bC_n$ is skew-symmetric, and $\bD_n$ is a damping matrix. Note again that the desired dynamics (\ref{eq:nom_dyn}) is defined in such a way that the robot can be  manipulated by human operator.  The purpose of the admittance control is to make the real dynamics (\ref{eq:real_dyn}) behave like (\ref{eq:nom_dyn}). To this end, the nominal dynamics in Fig. \ref{fig:control_str} is defined by
\begin{align}
\bM_n(\bq) \ddot{\bq}_n + \left( \bC_n(\bq,\dot{\bq}) + \bD_n \right) \dot{\bq}_n= \btau_h - \bK^{-1}\btau_a.
\label{eq:obs_dyn}
\end{align}

As shown in Fig. \ref{fig:control_str}, similar to (\ref{eq:adm_ctrl_1dof}), the admittance control input $\btau_a$ is given by
\begin{align}
\btau_a = \bK(\dot{\be}_{nr} + \bK_P \be_{nr}),
\label{eq:dob_input}
\end{align}
where $\be_{nr} \triangleq \bq_n - \bq$ is a vector representing the difference between nominal and real positions, and $\bK_P, \bK>0$ are diagonal gain matrices. 

Three main results are introduced in the following. First, let us begin with I/O stability.

\begin{thm}[$\mathcal{L}_2$ stability]
	Assume that the environment defines a passive map ($\dot{q}$ $\rightarrow$ $-\tau_e$). If no energy is stored in the system initially, the following $\mathcal{L}_2$ stability is satisfied:
	\begin{align}
	\int ||\bm{v}||_2^2 \leq \gamma \int ||\btau_h||_2^2,
	\label{eq:l2_stable}
	\end{align}
	where $\bm{v}=[\dot{\bq}^T \; \dot{\bq}_n^T]^T$, and $\gamma>0$ is a $\mathcal{L}_2$-gain.
	\label{thm:passivity}
\end{thm}
\begin{proof}
	See Section \ref{sec:theory_passivity}. 
\end{proof}

\begin{rem}
	Because the robot is controlled by the human operator, we cannot define the equilibrium point of the controlled system. Namely, we cannot say anything about stability in the sense of Lyapunov. Instead, we claim input-output $\mathcal{L}_2$ stability using the dissipativity of the closed-loop system. Physically speaking, as far as the input from human operator has finite energy (mathematically, $\btau_h \in \mathcal{L}_2$), the controlled system eventually dissipates  the energy obtained from the human operator; recall also Section \ref{sec:preliminary}. \QEDA
	\label{rem:phy_int_thm1}
\end{rem}

Although this theorem shows stability, the behavior of the controlled system is not clear yet. To solve the problem addressed in Section \ref{sec:problem_statement_and_related}, the proposed control scheme should be an admittance controller. To investigate this, we can rewrite (\ref{eq:real_dyn}) as follows after a lengthy algebraic derivation:
\begin{align}
 &\bM_n\ddot{\bq} + (\bC_n+ \bD_n)\dot{\bq} = \btau_h + \bM_n \bM^{-1} ( \btau_a -  \btau_{ed}),
\label{eq:real_rearrange}
\end{align}
where
\begin{align}
\nonumber  \btau_{ed}&=	\bM \bM_n^{-1}  \times \\
  &\; \Big(\widetilde{\bM} \bM^{-1}(-\bC\dot{\bq} + \btau_h+\btau_d) + \widetilde{\bC}\dot{\bq} - \bD_n\dot{\bq} - \btau_d\Big)
\label{eq:effective_dist}
\end{align}
is the effective disturbance, and $\widetilde{(\cdot)}= (\cdot) - (\cdot)_n$.

\begin{thm}[MS property]
	Let $\bK=\frac{1}{\epsilon} \bI$, where  $ 0< \epsilon \ll 1$. As $\epsilon \rightarrow 0$, the admittance control input $\btau_a$ exactly cancels the effective disturbance $\btau_{ed}$.
	\label{thm:dist_observation}
\end{thm}
\begin{proof} 
	See Section \ref{sec:theory_disturbance_observation}.
\end{proof}

\begin{figure}
	\centering
	\includegraphics[scale=0.36]{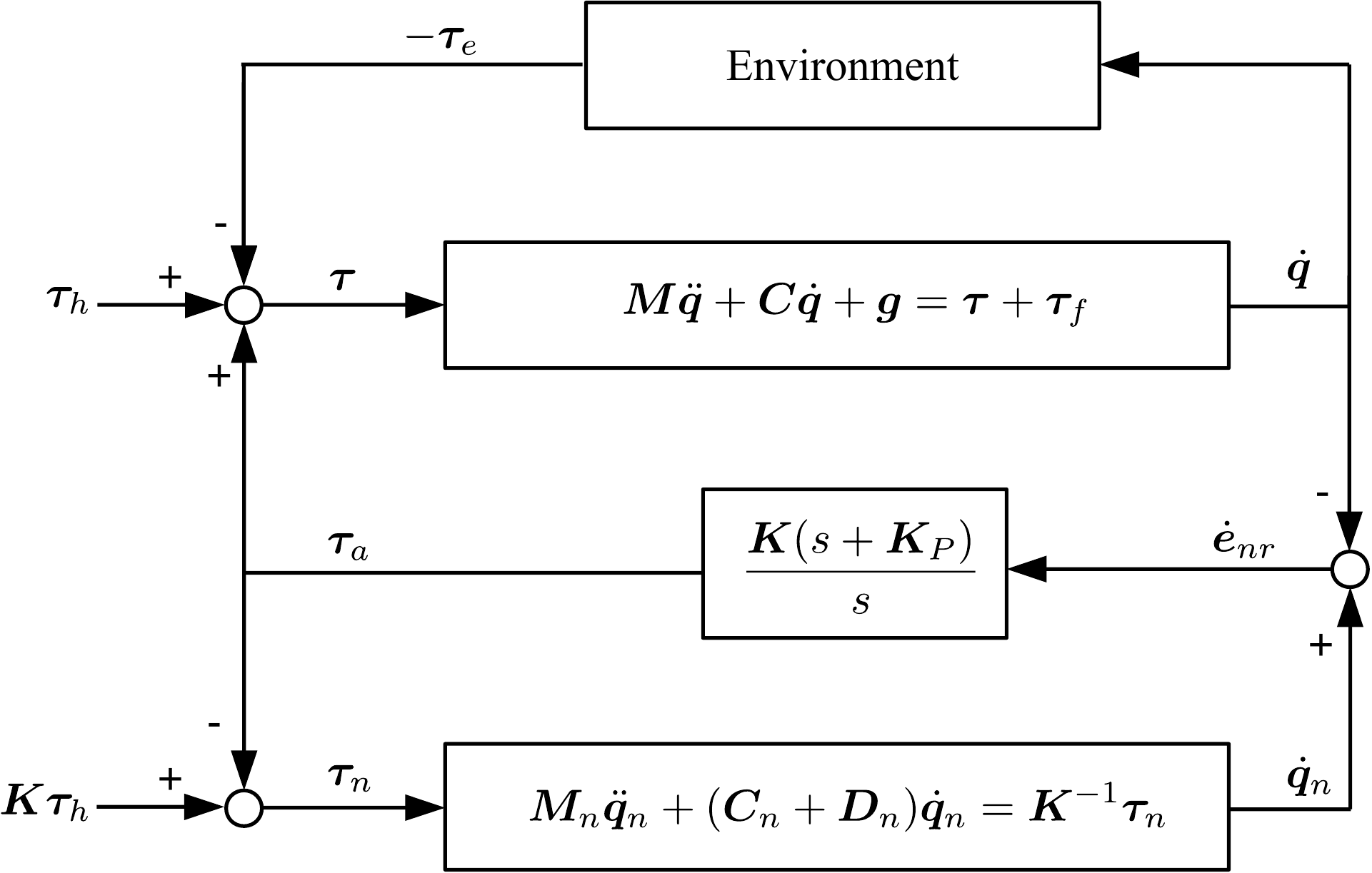}
	\caption{The structure of the proposed passivity-based nonlinear admittance controller. The closed-loop can be represented by feedback interconnections of (strictly) passive subsystems. Note that, similar to Fig. \ref{fig:adm_ctr_1dof}, the control structure is expressed using force (torque) and velocity, to show  passivity clearly.}
	\label{fig:control_str}
\end{figure}

This theorem indicates that the proposed control scheme is an admittance controller because the MS property is satisfied. However, the high gain setup (i.e., $\epsilon \rightarrow 0$) is not appealing in practice because such an arbitrarily small $\epsilon$ is not achievable for various reasons such as system bandwidth and noisy velocity measurements. Therefore, further analysis with practically achievable $\epsilon$ is desirable. The following theorem shows that the admittance controller is realized (in the sense of infinite norm) linearly with the control gain. We would like to remark that this analysis is not to show stability (which is already shown in Theorem \ref{thm:passivity}), but to predict the achievable performance limit.

\begin{thm}[Performance Limit]
	Let $\bK=\frac{1}{\epsilon}\bI$, where $0 < \epsilon \ll 1$. The achievable performance is limited as follows:
	\begin{align}
	|e_{nr}[i]|_\infty \leq ||\bx_{nr}||_\infty \propto \epsilon,
	\label{eq:pm_predic}
	\end{align}
	where $[i]$ denotes the $i$-th component of a vector, and $\bx_{nr}=[\be_{nr}^T \;\; \dot{\be}_{nr}^T]^T$.
	\label{thm:pm}
\end{thm}
\begin{proof}
	See Section \ref{sec:theory_performance}. 
\end{proof}

To summarize, the proposed approach can be realized by implementing the following two steps in every control loop.
\begin{enumerate}
	\item Update $\dot{\bq}_n$ and $\bq_n$ by integrating
		\begin{align}
			 \ddot{\bq}_n  = \bM_n^{-1}\left( - (\bC_n + \bD_n) \dot{\bq}_n +\btau_h - \bK^{-1}\btau_a   \right),
		\end{align}
		which is equivalent to (\ref{eq:obs_dyn}). Here note that $\btau_h$ is measured and $\btau_a$ is known.
	\item Apply the control input (\ref{eq:dob_input}) to follow $\bq_n$.
\end{enumerate}

\section{Theoretical Derivation}
\label{sec:theory}

\subsection{$\mathcal{L}_2$ stability (Proof of Theorem \ref{thm:passivity})} 
\label{sec:theory_passivity}

The controller (\ref{eq:dob_input}) satisfies output strict passivity as shown in  (\ref{eq:passivity_PI}). Due to the skew-symetricity of $\dot{\bM}-2\bC$ and $\dot{\bM}_n-2\bC_n$, the actual and nominal robots are also output strictly passive. Therefore, the overall control structure preserves passivity because it can be constructed by feedback interconnections of passive subsystems. As briefly addressed in Section \ref{sec:preliminary}, Lemma 6.8 (or Theorem 6.2) in \cite{khalil2002nonlinear} states that the feedback interconnection of the strictly passive blocks result in finite gain $\mathcal{L}_2$ stability. In particular, we obtain  (\ref{eq:l2_stable}) in our case.

\subsection{MS Property (Proof of Theorem \ref{thm:dist_observation})}
\label{sec:theory_disturbance_observation}

To begin with, we define the fast time scale $\sigma = t / \epsilon$. $\sigma$ is called ``fast time scale'' because it flows $1/\epsilon$ times faster than the actual time scale $t$ which is called ``slow time scale''. 

In state-space, the robot dynamics (\ref{eq:real_dyn}) can be expressed as
\begin{align}
\frac{d}{d\sigma}
\left(
\begin{array}{c}
\bq \\
\dot{\bq}
\end{array}
\right)
=
\left(
\begin{array}{c}
\epsilon \dot{\bq} \\
\epsilon \bM^{-1} ( -\bC\dot{\bq} + \btau_a + \btau_h + \btau_d )
\end{array}
\right).
\end{align}
Hence, as $\epsilon \rightarrow 0$, $
\frac{d}{d\sigma}
\left(
\begin{array}{c}
\bq \\
\dot{\bq}
\end{array}
\right)
=
\bzero
\label{eq:slow_var}
$. This implies that $\bq$, $\dot{\bq}$ are frozen in the fast time scale $\sigma$.

Next, we define the difference dynamics, which is obtained by subtracting (\ref{eq:real_dyn}) from (\ref{eq:obs_dyn}), as follows:
\begin{align}
\nonumber &\bM (\ddot{\be}_{nr}+\bK_P\dot{\be}_{nr})\\
 & + \bM \bM_n^{-1}(\bC_n+\bD_n+\bM_n \bM^{-1} \bK) (\dot{\be}_{nr} +\bK_P \be_{nr})= \btau_{ed}.
\label{eq:dif_dyn}
\end{align}
Using $(\dot{\be}_{nr}+\bK_P \be_{nr})= \bK^{-1}\btau_a= \epsilon \btau_a$, we have
\begin{align}
\nonumber &\bM_n \frac{d}{d\sigma}\btau_a + (\epsilon(\bC_n+\bD_n)+\bM_n \bM^{-1}) \btau_a = \\
&\;\; -\btau_d + \widetilde{\bM} \bM^{-1}(-\bC\dot{\bq} +\btau_h + \btau_d) + \widetilde{\bC} \dot{\bq} - \bD_n \dot{\bq}.
\end{align}
Hence, as $\epsilon \rightarrow 0$,
\begin{align}
\bM \frac{d}{d\sigma}\btau_a + \btau_a = \btau_{ed}.
\label{eq:boundary_layer}
\end{align}
Because $\bq$ and $\dot{\bq}$ are frozen variables, the right-hand side of (\ref{eq:boundary_layer}) and $\bM^{-1}=\bM(\bq)^{-1}>0$ can be considered as constant values. Therefore, in the fast time scale $\sigma$, $\btau_a$ exponentially converges to the effective disturbance defined in (\ref{eq:effective_dist}). 

Physically speaking, because $\btau_a$ converges to the effective disturbance exponentially fast with respect to $\sigma$, $\btau_a$ can be regarded as the effective disturbance itself in the slow time scale $t$ (i.e., very fast transient).  As a result, as $\epsilon$ goes to zero, the system behaves as the desired dynamics $\bM_n\ddot{\bq}+(\bC_n+\bD_n)\dot{\bq}=\btau_h$ in the slow time scale $t$.

\subsection{Performance Limit (Proof of Theorem \ref{thm:pm})}
\label{sec:theory_performance}

To begin with, let us alternatively express (\ref{eq:dif_dyn}) as
\begin{align}
\bM(\ddot{\be}_{nr} + \bK_P \dot{\be}_{nr}) + (\bC+\bD_n)(\dot{\be}_{nr}+\bK_P \be_{nr}) = -\btau_a+\bw,
\label{eq:diff_dyn_alter}
\end{align}
where $\bw =\widetilde{\bM}\bM_n^{-1} ( -(\bC_n+\bD_n)(\dot{\bq}_n + \bK_P \be_{nr})+\btau_h ) + \widetilde{\bC}(\dot{\bq}_{n} + \bK_P \be_{nr})-\btau_d$. In state-space form, (\ref{eq:diff_dyn_alter}) can be expressed as
\begin{align}
\dot{\bx}_{nr} = \bA \bx_{nr} + \bB(-\btau_a) + \bB\bw,
\end{align}
where
\begin{align}
\nonumber \bA&=
\left[
\begin{array}{cc}
\bzero & \bI \\
-\bM^{-1}(\bC+\bD_n)\bK_P & -\bM^{-1}(\bC+\bD_n)-\bK_P
\end{array}
\right], \\
\bB&=
\left[
\begin{array}{c}
\bzero \\
\bM^{-1}
\end{array}
\right].
\end{align}
For this system, we measure the performance using the worst-case value of $|e_{nr}[i]|$.

Because the closed-loop is dissipative (Theorem \ref{thm:passivity}), the boundedness of $\bw$ is already guaranteed for bounded $\btau_h$. However, for tighter analysis, we bound $||\bw||^2$ by a polynomial of $\bx_{nr}$ as follows:
\begin{lem}
$||\bw||^2$ is bounded from above by 
\begin{align}
||\bw||^2 \leq b_1 ||\bx_{nr}||^2 + b_2 ||\bx_{nr}|| + b_3 ,
\label{eq:w_sq_boundedness}
\end{align}
for some constants $b_1,b_2,b_3>0$.
\end{lem}
\begin{proof}
$\bw$ is expressed using not only $\bq, \dot{\bq}, \be_{nr}$, and $\dot{\be}_{nr}$, but also $\bq_n$ and $\dot{\bq}_n$. To eliminate redundancy, we use the relation $\bq_n=\be_{nr}+\bq,\dot{\bq}_n=\dot{\be}_{nr}+\dot{\bq}$ so that $\bw$ can be expressed as $\bw=\widetilde{\bM}\bM_n^{-1}(-(\bC_n+\bD_n)(\dot{\be}_{nr} + \bK_P \be_{nr}+\dot{\bq})+\btau_h) + \widetilde{\bC}(\dot{\be}_{nr} + \bK_P \be_{nr}+\dot{\bq})$. Because $\bq$ and $\dot{\bq}$ are bounded, $\bC_n$ and $\widetilde{\bC}$ are also bounded. Therefore, $||\bw||$ can be bounded by a first-order polynomial of $||\bx_{nr}||$, and as a result, (\ref{eq:w_sq_boundedness}) follows immediately.
\end{proof}

Next, we define a Lyapunov-like function $V=\frac{1}{2}\bx_{nr}^T \bP \bx_{nr}$, where
\begin{align}
\bP=
\left[
\begin{array}{cc}
\bM \bK_P^2 + \frac{1}{\epsilon} \bK_P & \bK_P \bM \\
\bM \bK_P & \bM
\end{array}
\right].
\end{align}
In addition, define a matrix $\bQ>0$ as $\bQ=diag\left\{  \bK_P^2/\epsilon,  \bI/\epsilon  \right\}$.
Then, the following equation is satisfied:
\begin{align}
\dot{\bP} + \bA^T\bP + \bP\bA - \frac{1}{\epsilon} \bP\bB \bB^T\bP + \bQ = \bzero.
\label{eq:riccati}
\end{align}

The time derivative of $V$ is $\dot{V} = \frac{1}{2}\bx_{nr}^T (\dot{\bP} + \bA^T\bP + \bP\bA)\bx_{nr} - \bx_{nr}^T \bP\bB \btau_a + \bx_{nr}^T \bP\bB \bw$.
Using (\ref{eq:w_sq_boundedness}), (\ref{eq:riccati}), and $\btau_a=\bK(\dot{\be}_{nr}+\bK_P \be_{nr})=\frac{1}{\epsilon}\bB^T \bP \bx_{nr}$, 
\begin{align}
\nonumber \dot{V} &= -\frac{1}{2} \bx_{nr}^T \bQ \bx_{nr} - \frac{1}{2\epsilon}||\bB^T \bP \bx_{nr} - \epsilon \bw ||^2 + \frac{\epsilon}{2} ||\bw||^2 \\
 &\leq -\frac{1}{2} (\bQ-\epsilon b_1 \bI) ||\bx_{nr}||^2 + \frac{\epsilon}{2} b_2 ||\bx_{nr}|| + \frac{\epsilon}{2}  b_3.
\end{align}
If $\bK_P>\bI$, then the minimum eigenvalue of $\bQ$ is $\frac{1}{\epsilon}$.\footnote{The case $\bK_P \ngtr \bI$ case can be analyzed similarly}  Thus,
\begin{align}
\dot{V} \leq -\frac{1}{2} \lambda_\epsilon ||\bx_{nr}||^2 + \frac{\epsilon}{2} b_2 ||\bx_{nr}|| + \frac{\epsilon}{2}  b_3,
\label{eq:V_dot_upper}
\end{align}
where $\lambda_{\epsilon}=(\frac{1}{\epsilon}-\epsilon b_1)$. Although the true shape of $\dot{V}$ is not known, we know that it is upper bounded by (\ref{eq:V_dot_upper}). From this, we can analyze the worst case norm of the state $\bx_{nr}$. To avoid misunderstanding, recall that the goal of this analysis to predict the worst case performance, while the stability is already guaranteed by Theorem \ref{thm:passivity}.

Because $\epsilon$ will be chosen small so that $\lambda_\epsilon \simeq \frac{1}{\epsilon}>0$, an upper bound for $||\bx_{nr}||$ can be obtained by solving $(\ref{eq:V_dot_upper})=0$ as follows:
\begin{align}
\nonumber ||\bx_{nr}||_\infty \leq& \frac{\epsilon b_2 + \sqrt{\epsilon^2 b_2^2 + 4b_3}}{2/\epsilon}  = \frac{b_2}{2}\epsilon^2 + \frac{\epsilon}{2}\sqrt{\epsilon^2b_2^2 + 4b_3} \\
	          \leq& a_1 \epsilon^2 + a_2 \epsilon,
\end{align}
for some constants $a_1,a_2>0$. Finally, because $\epsilon$ is chosen small (namely, $\epsilon$ dominates $\epsilon^2$) and $|e_{nr}[i]|_\infty \leq ||\bx_{nr}||_\infty$, the relation (\ref{eq:pm_predic}) follows.

\begin{figure}[]
\centering
{\includegraphics[scale=0.27]{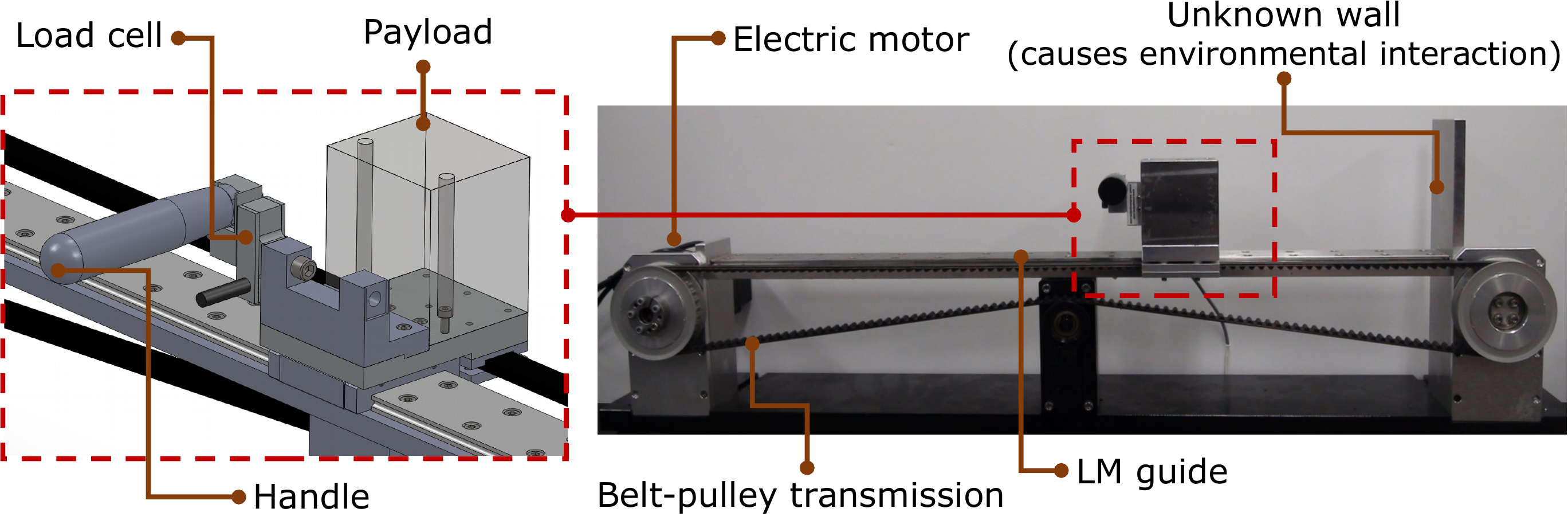}}
\caption{1 DoF testbed for validation in Section \ref{sec:exp_1dof}.}
\label{fig:1dof_testbed}
\end{figure}

\section{Experimental Validation using 1 DoF Testbed}
\label{sec:exp_1dof}

This section presents the validation of the proposed approach using 1 DoF testbed shown in Fig. \ref{fig:1dof_testbed}. This testbed is designed to emulate the scenario of interest introduced in Section \ref{sec:intro}.	A human operator grasps the handle to move the mass, and the interaction force $\tau_h$ is measured by the load cell. The mass of the testbed itself was approximately 1 $\mathrm{kg}$ and that of the payload was approximately 5 $\mathrm{kg}$. The environmental interaction occurred when the human operator pushes the mass toward the wall. We would like to remark that (i) the mass and (ii) environmental interaction were unknown to the controller. The proposed method was implemented using a real-time OS (RTX) with 500 Hz control frequency.

%

\begin{figure}[]
	\centering
	\includegraphics[scale=0.31]{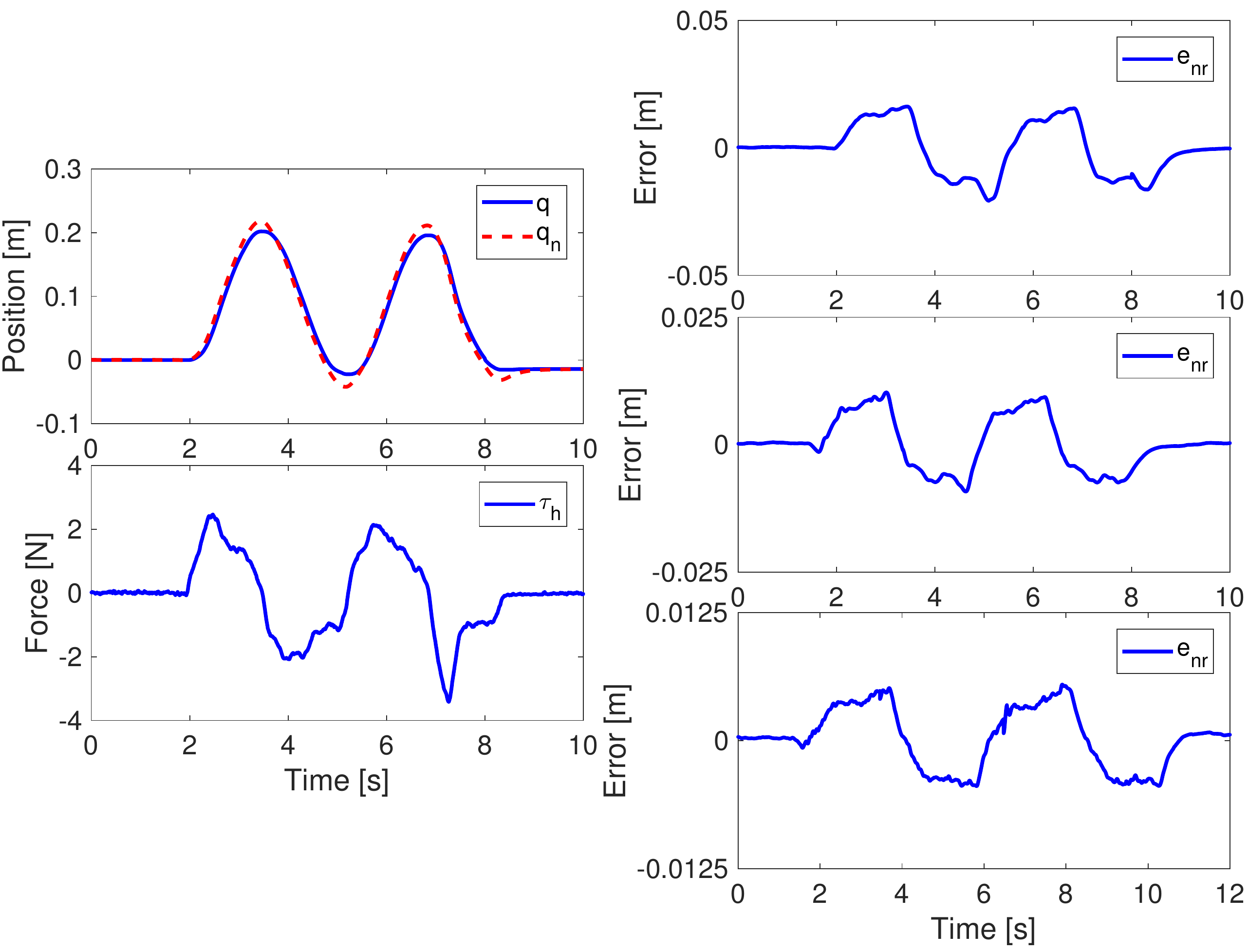}
	\caption{Gain tuning experiment using 1 DoF testbed. Left column: $q$, $q_n$ and $\tau_h$ with $K=7.5$. Human operator moved the mass sinusoidally in free space. Right column: $K=7.5$ (top), $K=15$ (middle), and $K=30$ (bottom). As the control gain was doubled (i.e., $\epsilon$ was halved), $|e_{nr}|_\infty$ decreased by half; note that the $y$ axis scale is halved as $\epsilon$ decreases by half. $|e_{nr}|_\infty$ was decreased by $0.0207 \rightarrow 0.0103 \rightarrow 0.00549$ $\mathrm{m}$.}
	\label{fig:1dof_perf}
\end{figure}

\begin{figure}[]
	\centering
	\subfigure
	{\includegraphics[scale=0.3]{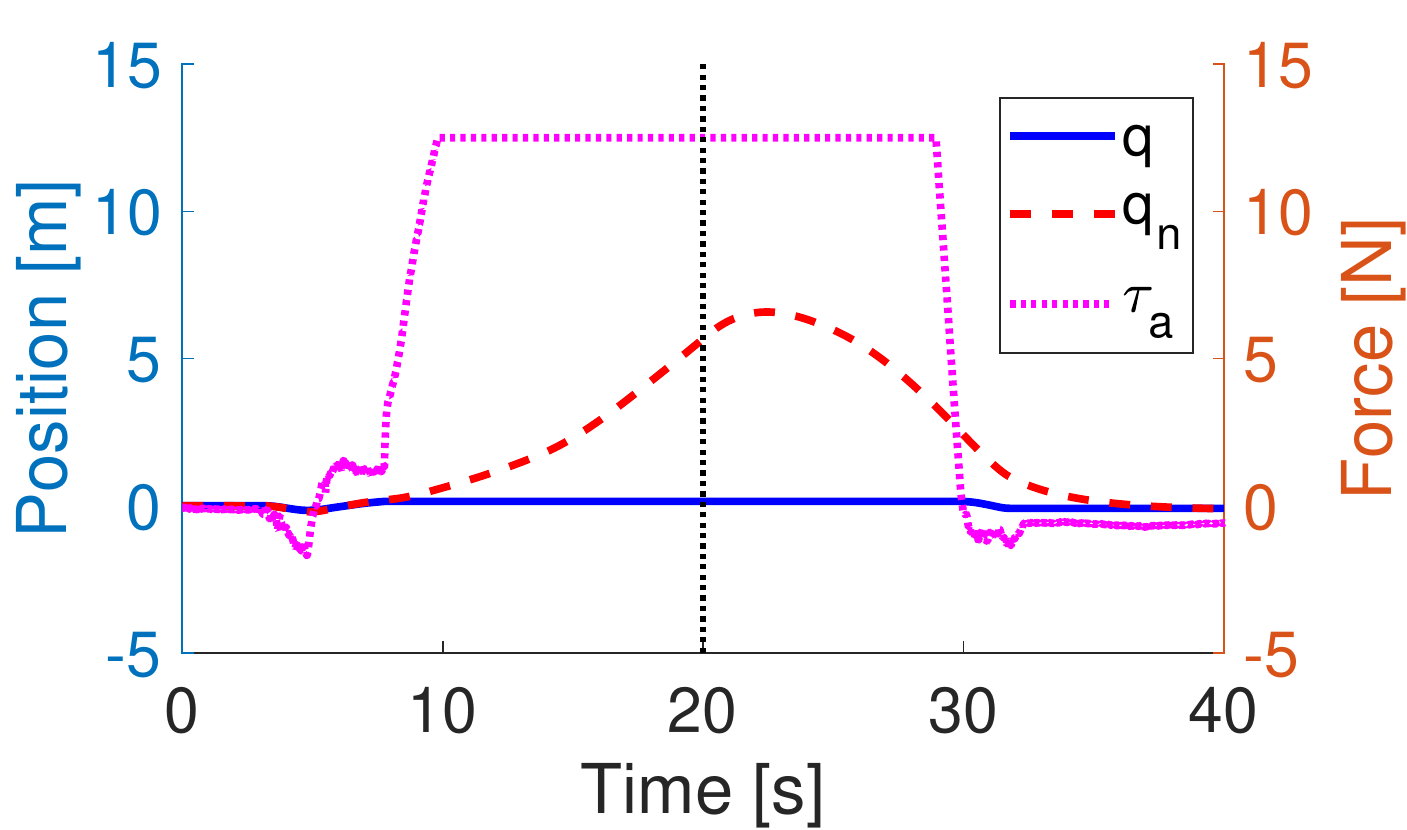}}
	\centering
	\subfigure
	{\includegraphics[scale=0.3]{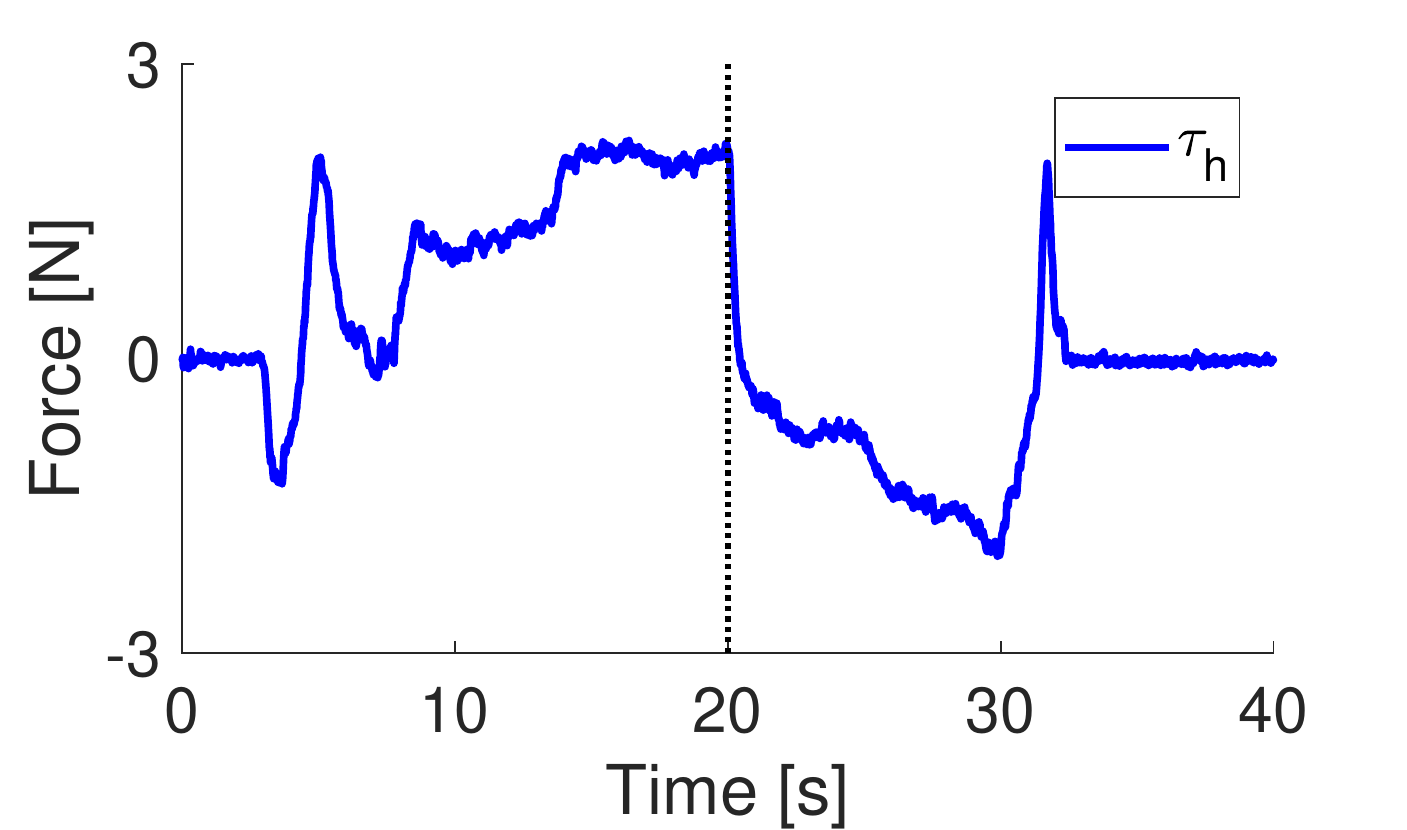}}
	\caption{Standard admittance controller was applied under the unknown environmental interaction (recall the motivating example). 0-10 $\mathrm{sec}$: moves in free space. 10-20 $\mathrm{sec}$: intentional contact with the wall (which is not known to the controller). 20-30 $\mathrm{sec}$: tried to detach the robot from the wall, but could not because the admittance control input is pushing the robot toward the wall (control input $\tau_a$ is saturated at 12.5 $\mathrm{N}$). After 30 $\mathrm{sec}$: nominal signal $q_n$ is reduced, and the robot could move back to the free space.}
	\label{fig:exp_1dof_standard}
\end{figure}

\begin{figure}[t]
	\centering
	\subfigure[$M_n=0.1$ $\mathrm{kg}$]
	{\includegraphics[scale=0.28]{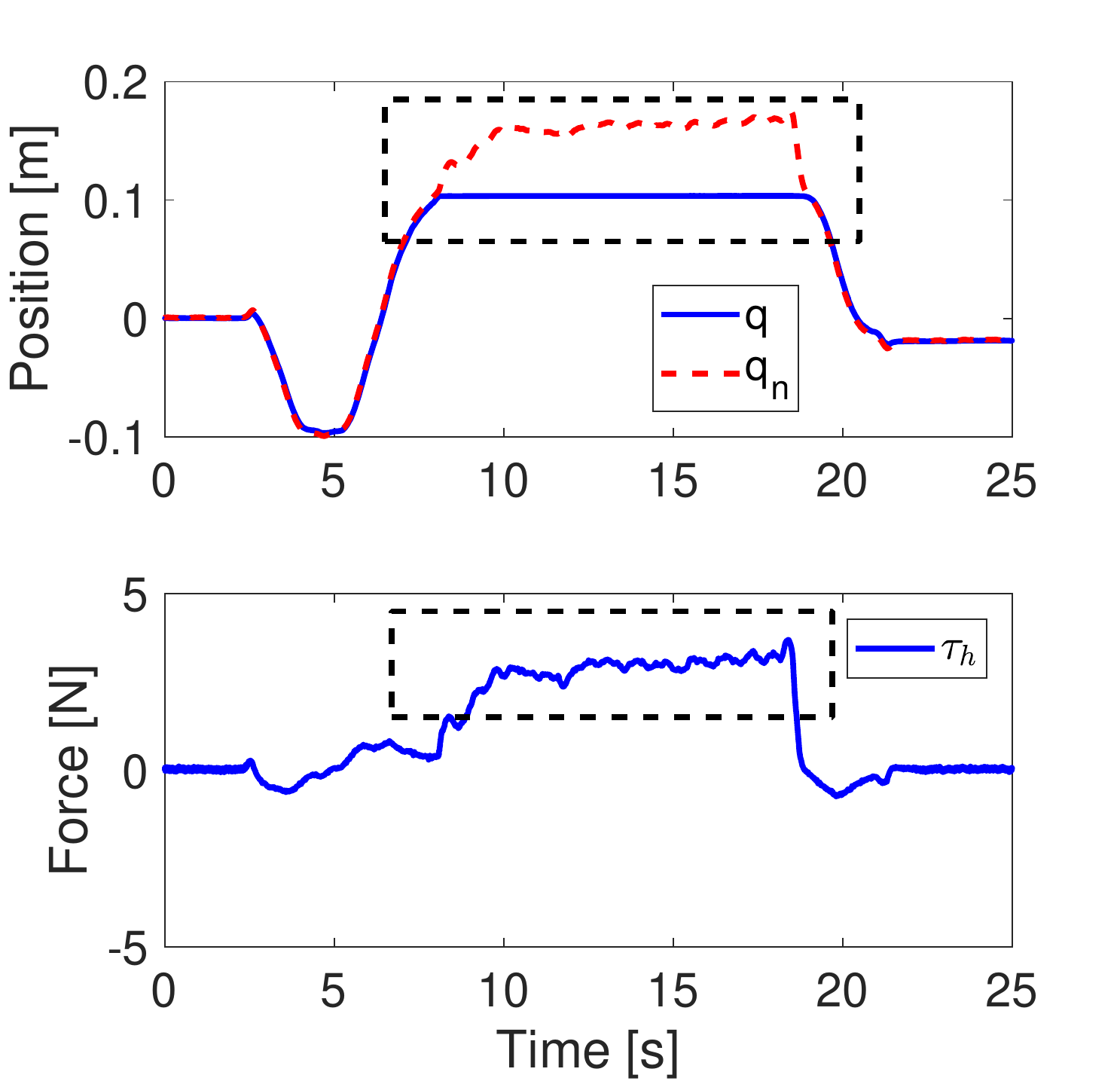}}
	\centering
	\subfigure[$M_n=10$ $\mathrm{kg}$]
	{\includegraphics[scale=0.28]{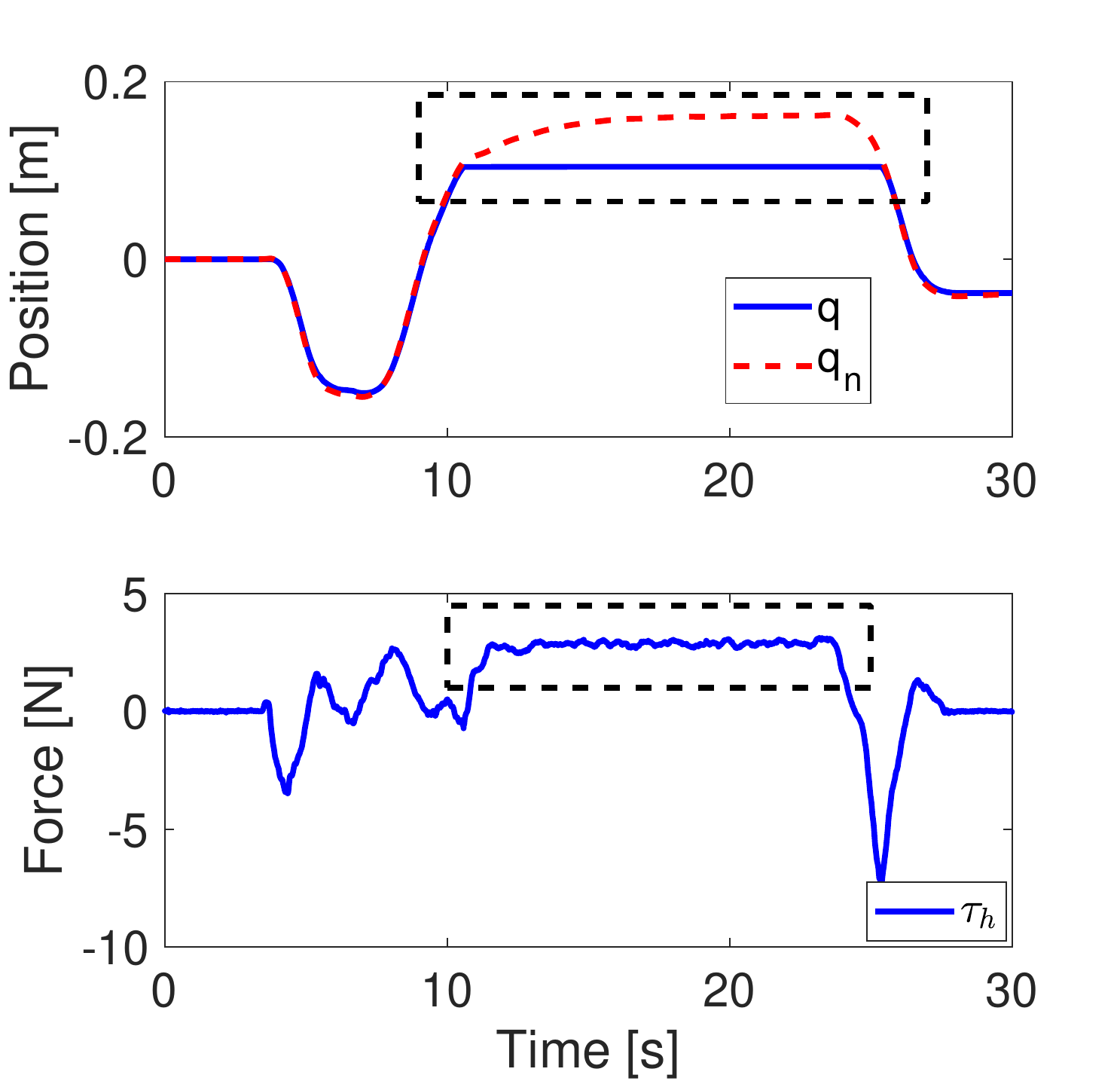}} 
	\caption{The proposed approach was applied under the unknown environmental interaction. Top row: $q$, $q_n$. Bottom row: interaction force $\tau_h$ measured using the load cell. Left column: $M_n=0.1$ $\mathrm{kg}$. Right column: $M_n=10$ $\mathrm{kg}$. Black dashed box shows that the nominal signal did not keep increasing although the force was constantly applied. } 
	\label{fig:1dof_passivity}
\end{figure}

\subsection{Performance limit}
\label{section4-1}

The performance was tuned according to Theorem \ref{thm:pm}. In the experiments, the human operator moved the mass sinusoidally in the free space (i.e., without environmental interaction) while changing the control gain. Letting $K_P=10$, $\epsilon$ was changed as $0.133 \rightarrow 0.067 \rightarrow 0.033$ (corresponding $K$ was $7.5 \rightarrow 15 \rightarrow 30$). Fig. \ref{fig:1dof_perf} shows human interaction forces and corresponding behaviors for each gain.\footnote{Force sensor signals are low pass filtered for better presentation (in every figure in this paper). However, raw signals were used in real implementations.} As expected from Theorem \ref{thm:pm}, the performance (in the sense of $|e_{nr}|_\infty$) is improved linearly with the control gain; $|e_{nr}|_\infty$ was decreased by $0.0207 \rightarrow 0.0103 \rightarrow 0.00549$ $\mathrm{m}$. Note that the $y$-axis of the right column of Fig. \ref{fig:1dof_perf} is halved together with $\epsilon$.

\subsection{Wall sticking effect}
\label{sec:wall_sticking}

To investigate the wall sticking effect, experiments with the following steps were performed. (i) Human operator applied sinusoidal force in free space; (ii) human operator pushed the mass toward the wall, and maintained contact for a while; (iii) human operator applied force in the opposite direction, to move the mass back into the free space.

As shown in Fig. \ref{fig:exp_1dof_standard}, the standard admittance controller suffers from wall sticking effect even if the real system parameter ($M_n=M=5\mathrm{kg}$) is known. At the beginning ($0$ to $10$ $\mathrm{sec}$), the mass moved according to the human operator's command. From $10$ to $20$ $\mathrm{sec}$, human operator pushed the mass into the wall to generate unknown environmental interaction. During this period, $q_n$  deviated from $q$, and the admittance control input is saturated because of large $e_{nr}$. From $20$ to $30$ $\mathrm{sec}$, the human operator tried to move back the mass from the wall, but the mass could not move because the operator had to wait until $q_n$ to decrease to a reasonable value. Note that, in this period, the admittance control input was still saturated. To summarize, the standard admittance controller showed wall-sticking effect under the unexpected interaction.

To show that the proposed approach reduces the wall sticking effect regardless of the parameters, it was implemented with two different $M_n$ while letting $D_n=0.5$ $\mathrm{Ns/m}$: $0.1$ and $10$ $\mathrm{kg}$ (the real mass was $5\mathrm{kg}$). As shown in Fig. \ref{fig:1dof_passivity}, for both cases, $q_n$ converged to a certain equilibrium due to the spring (recall the motivating example). Namely, the wall sticking effect was reduced.

\section{Powered Upper-Limb Control}
\label{sec:powered}

\subsection{Task Space Control}
\label{sec:task_space_control}

Although the generalized coordinates $\bq$ is used in the previous, the formulation in Cartesian space is more appealing in our application. The Cartesian space variable $\bp$ can be obtained using the Jacobian matrix $\bJ(\bq)$:
\begin{align}
\dot{\bp}=\bJ(\bq) \dot{\bq}.
\end{align}
Here, $\bp$ is the position vector (in our formulation, $\bp=[p_x \;\; p_y \;\; p_z]^T \in \Re^3$ represents the position in each direction).

The robot dynamics can be expressed in Cartesian space as
\begin{align}
\bLambda (\bq) \ddot{\bp} + \bGamma (\bq,\dot{\bq}) \dot{\bp} + \bzeta (\bq) = \bff_a + \bff_h + \bff_f + \bff_e,
\end{align}
where $\bLambda$, $\bGamma$, and $\bzeta$ represent the inertia, the Coriolis/centrifugal force, and the gravity force, respectively. Note that the control commands in Cartesian space are transformed into the joint space using the Jacobian transpose. Namely, 
\begin{align}
 \btau_a = \bJ^T(\bq) \bff_a = \bJ^T(\bq) \bK (\dot{\be}_{nr} + \bK_P \be_{nr}).
\end{align}
Here, $\be_{nr}=\bp_n - \bp$ is redefined for simplicity of writing.

The nominal dynamics is
\begin{align}
\bLambda_n \ddot{\bp}_n + \bD_n \dot{\bp}_n  = \bff_h - \bK^{-1} \bff_a.
\label{eq:obs_dyn_cart}
\end{align}
In our implementation, the parameters were defined by $\bLambda_n=diag\{5,5,5\}$ $\mathrm{kg}$ and $\bD_n = diag\{ 20, 20, 20 \}$  $\mathrm{kg/s}$, so the behaviors in the $x$-, $y$-, and $z$- axes are identical and decoupled.

\begin{figure}[]
	\centering
	{\includegraphics[scale=0.38]{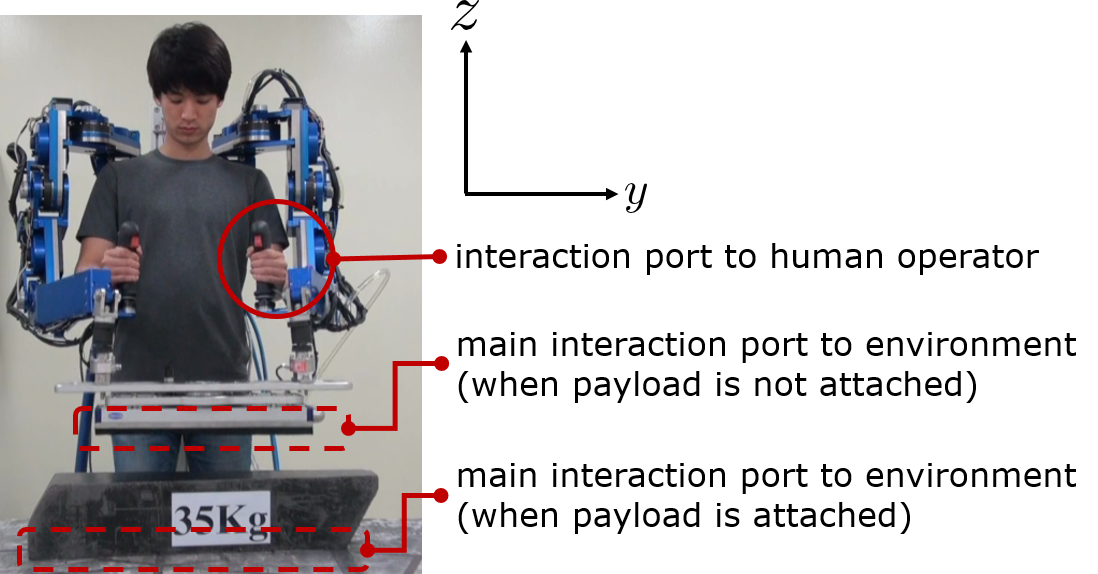}}\\
	\caption{Experimental setup for validation in Section \ref{sec:task_space_control}.}
	\label{fig:exp_setup}
\end{figure}

\subsection{Experimental Setup}

An actual powered upper-limb robot shown in Fig. \ref{fig:exp_setup} was used to evaluate the proposed approach. Recall Fig. \ref{fig:hardware}b for hardware details. The robot is supposed to lift and maneuver a payload with unknown mass. The robot carries the payload using the suction pad attached to the end-effectors. When attaching/detaching the payload, interactions with the environment are inevitable because the human operator should command the robot to move towards the environment. Recall that the environmental interaction is unknown to the controller.


\subsection{Experimental Results}

\begin{figure*}[]
	\centering
	\subfigure[]
	{\includegraphics[scale=0.39]{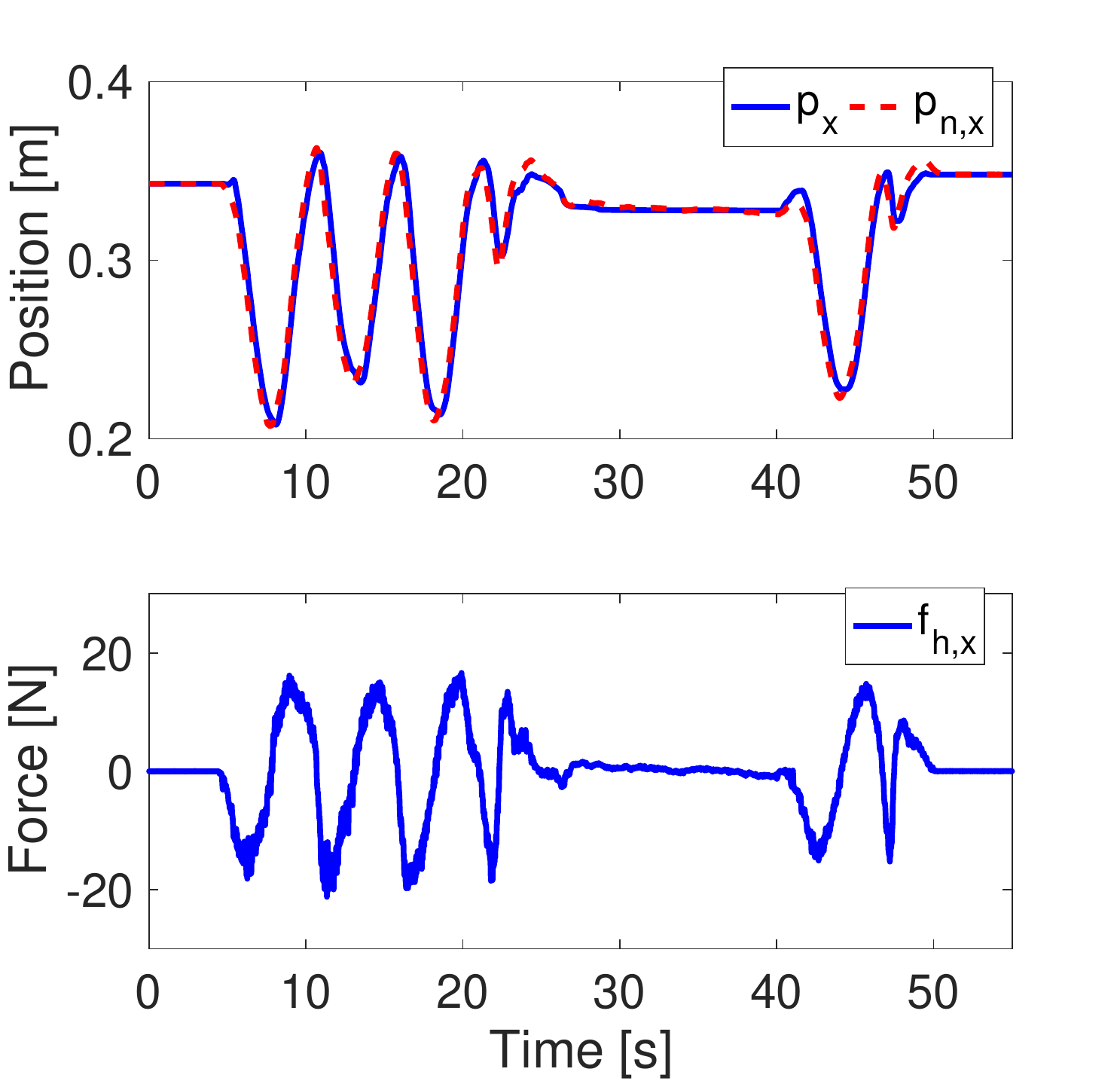}}
	\centering
	\subfigure[]
	{\includegraphics[scale=0.39]{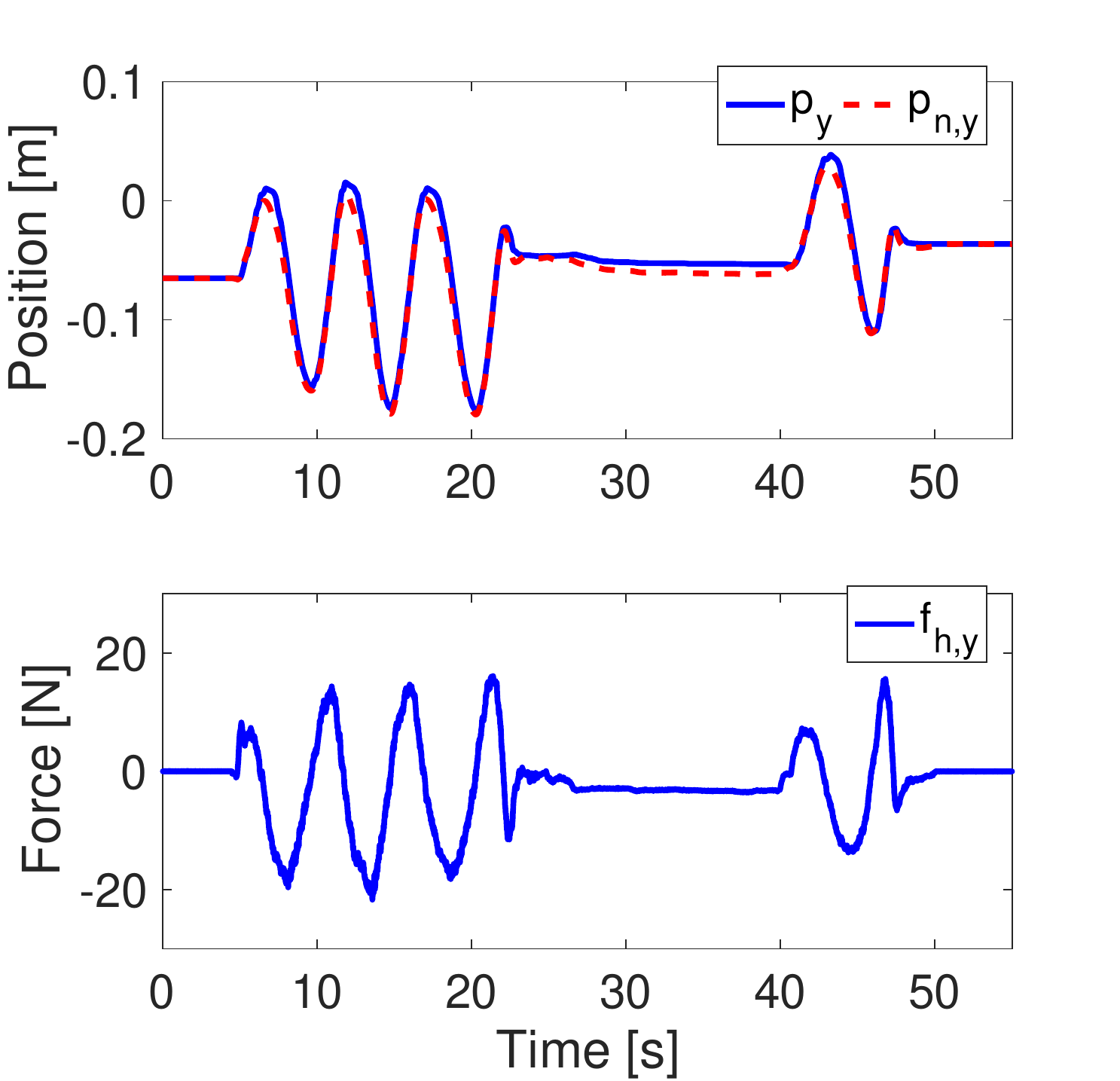}} 
	\centering
	\subfigure[]
	{\includegraphics[scale=0.39]{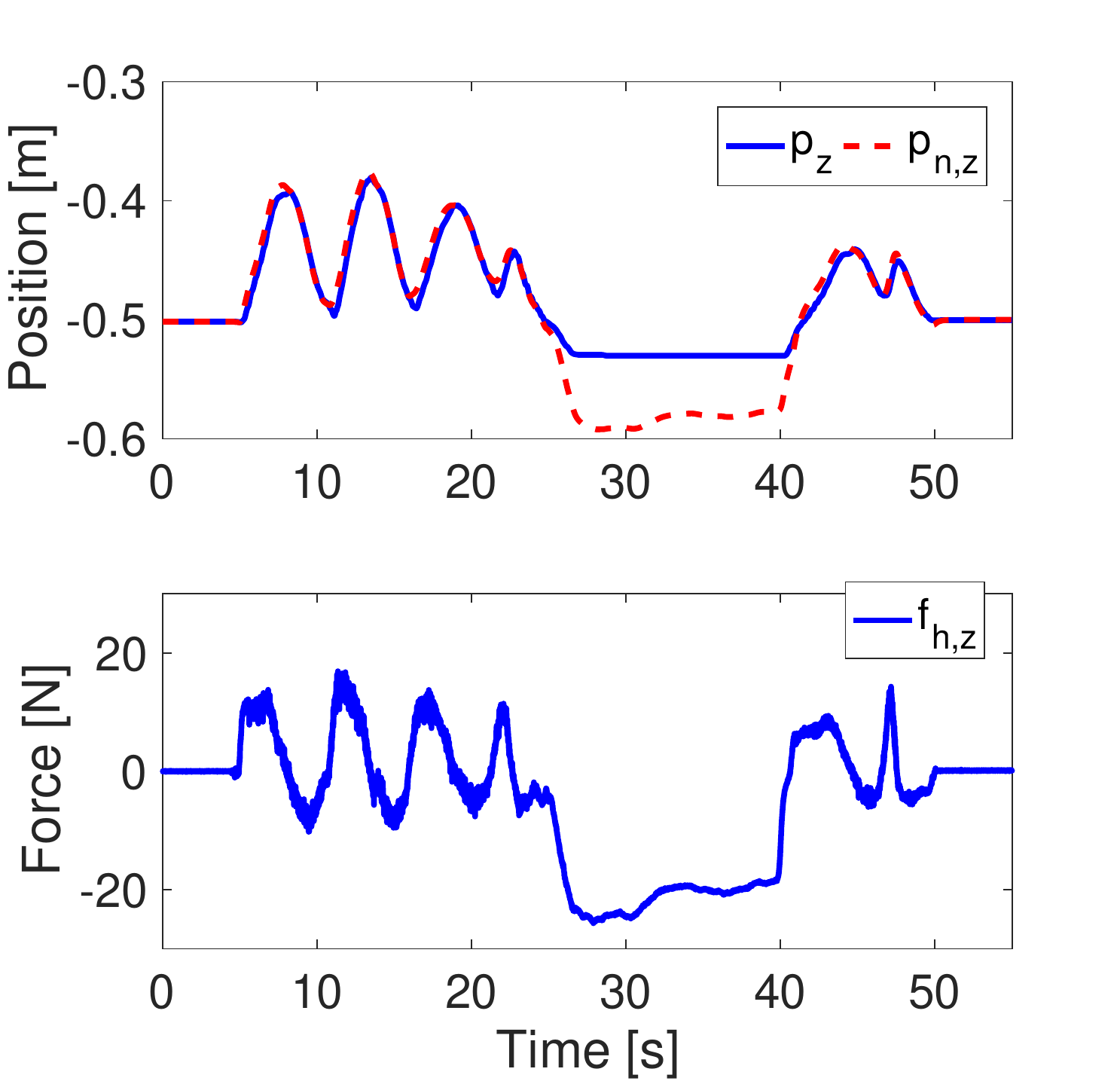}} 
	\caption{The robot interacts with not only the human operator (5-50 $\mathrm{sec}$), but also with the (unknown) environment (25-40 $\mathrm{sec}$). The top row shows $\bp$ and $\bp_n$, and the bottom row shows the applied force $\bff_h$. }
	\label{fig:exp_passivity}
\end{figure*}

Three experiments were performed for verification. Although two arms (left and right) are operated concurrently, the results for only the left arm will be presented because those for the right arm are similar.

\begin{figure*}[]
\subfigure[]
{\includegraphics[scale=0.39]{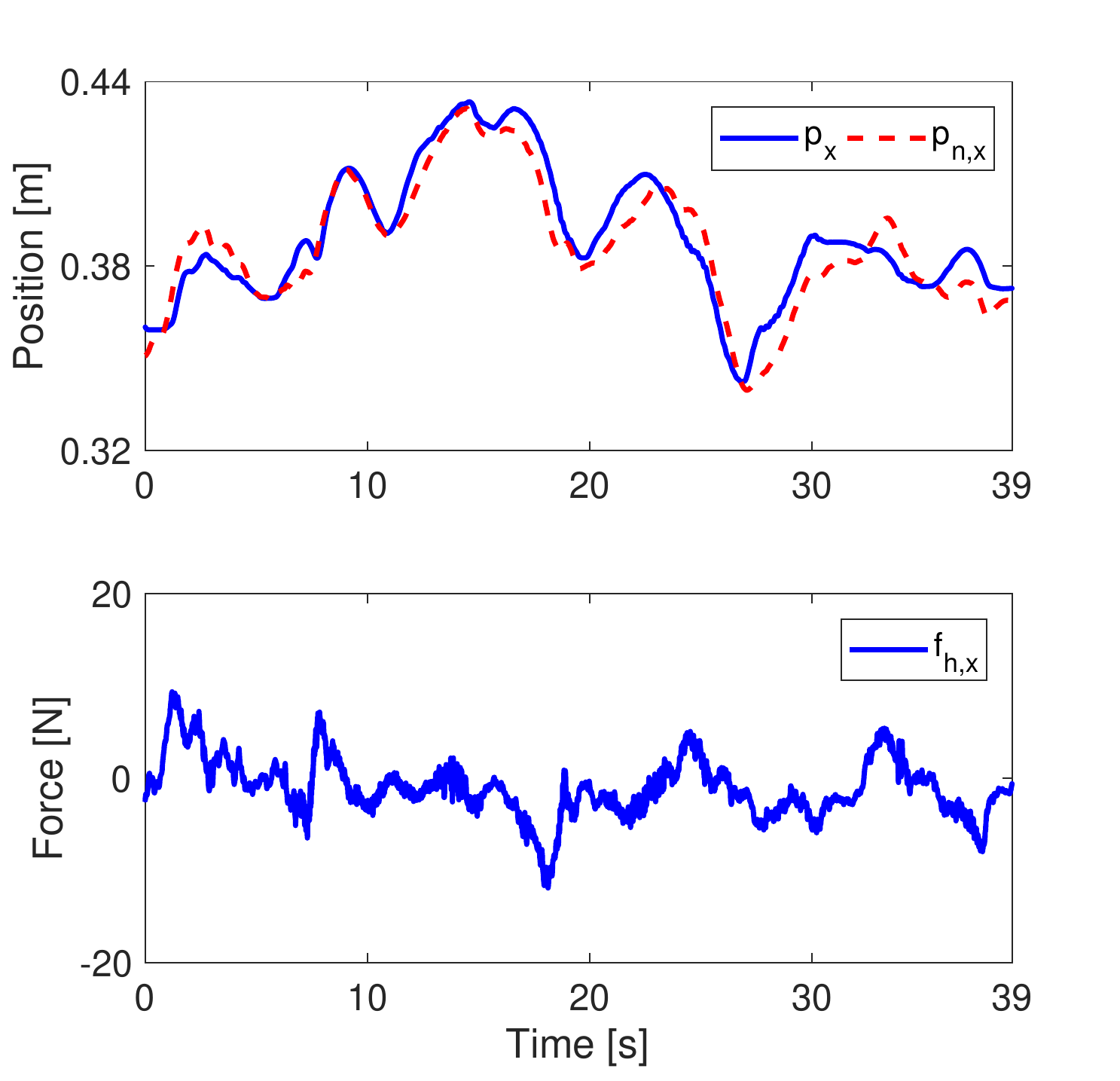}}
\centering
\subfigure[]
{\includegraphics[scale=0.39]{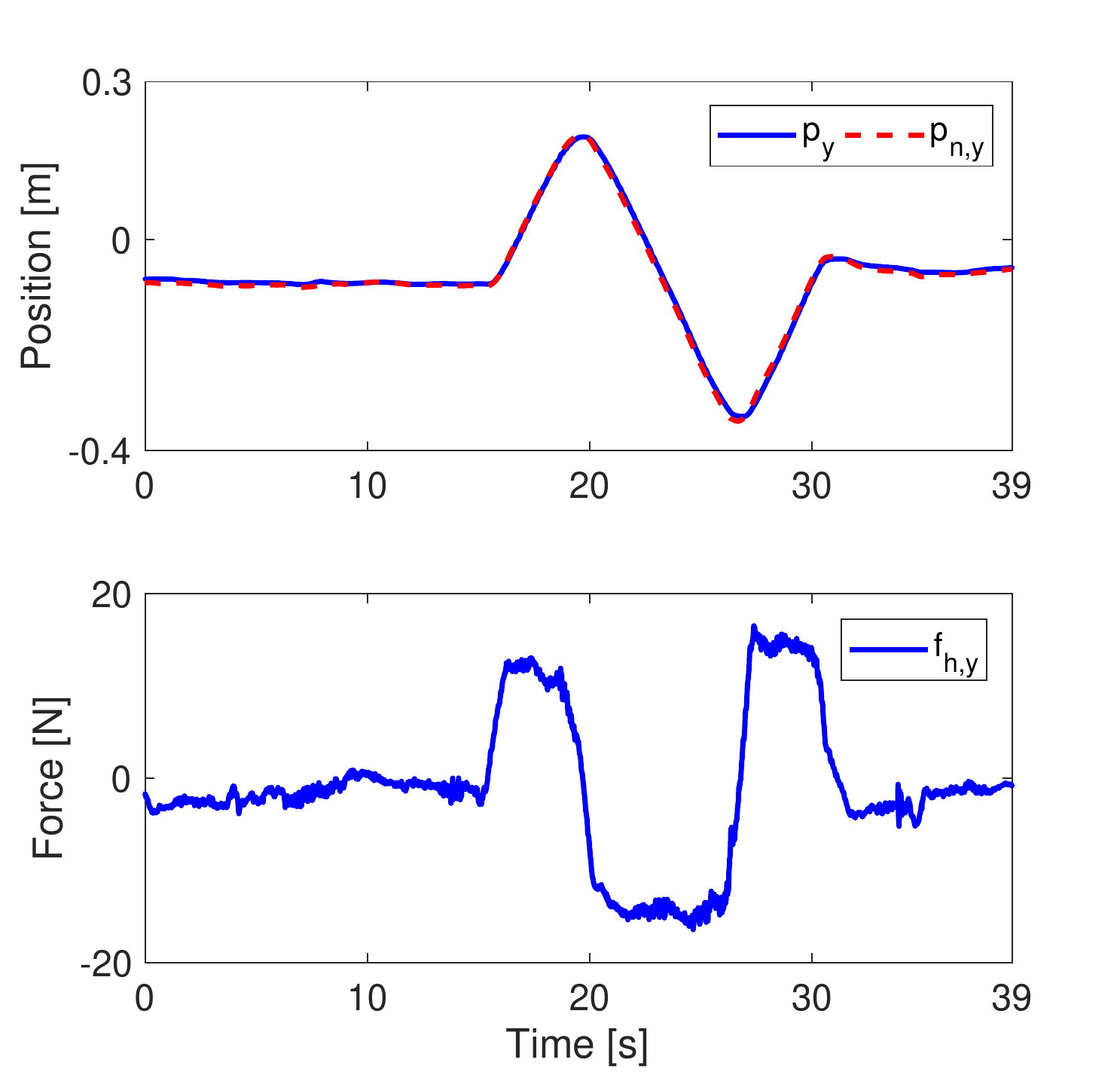}} 
\centering
\subfigure[]
{\includegraphics[scale=0.39]{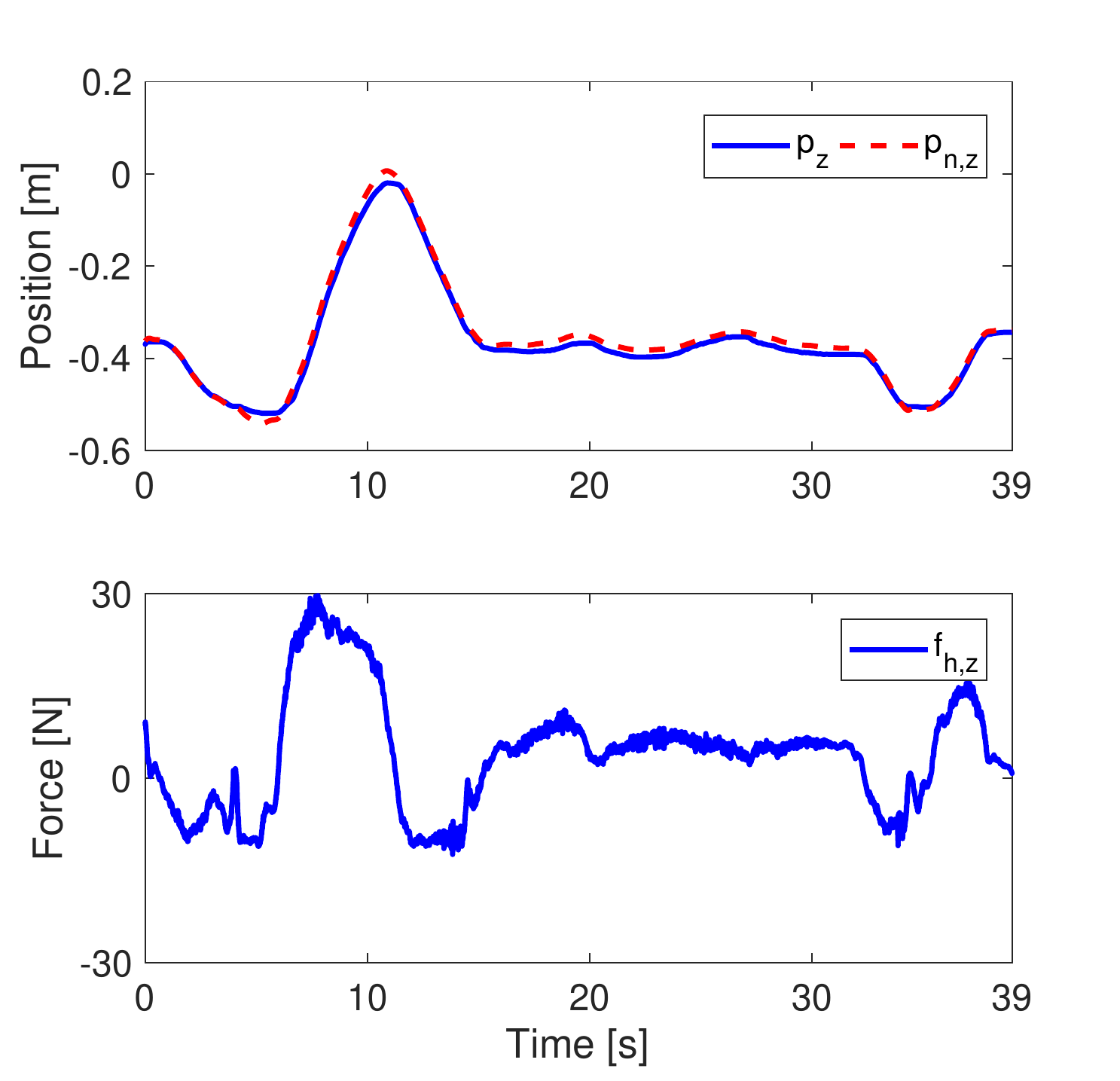}} \\
\centering
\subfigure[initial position]
{\includegraphics[scale=0.65]{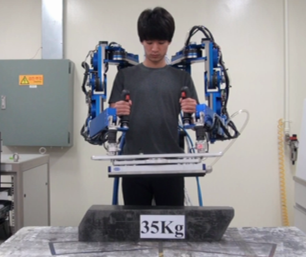}}
\centering
\subfigure[attaching a payload]
{\includegraphics[scale=0.65]{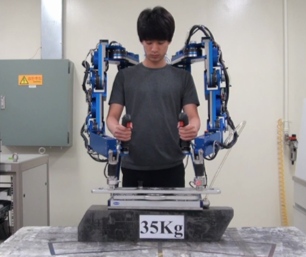}}
\centering
\subfigure[lifting up]
{\includegraphics[scale=0.65]{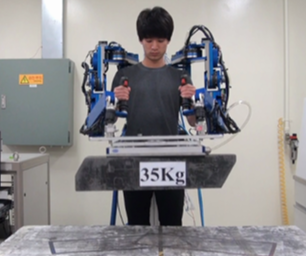}} 
\centering
\subfigure[lifting high]
{\includegraphics[scale=0.65]{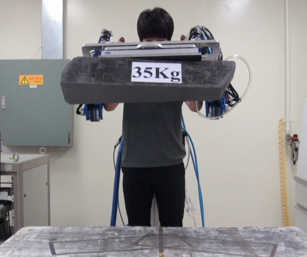}} 
\centering
\subfigure[moving to center position]
{\includegraphics[scale=0.65]{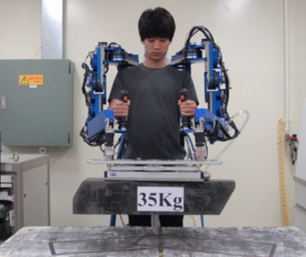}} \\
\centering
\subfigure[moving left]
{\includegraphics[scale=0.65]{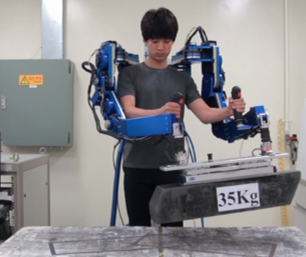}}
\centering
\subfigure[moving right]
{\includegraphics[scale=0.65]{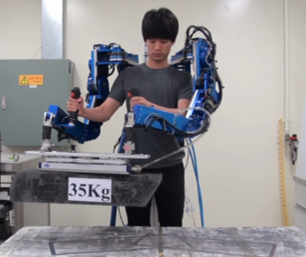}} 
\centering
\subfigure[moving to center position]
{\includegraphics[scale=0.65]{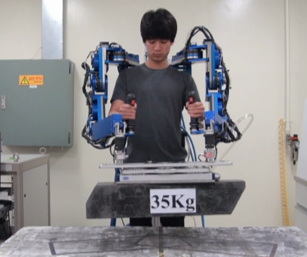}} 
\centering
\subfigure[detaching a payload]
{\includegraphics[scale=0.65]{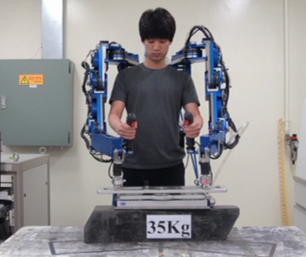}}
\centering
\subfigure[return to initial position]
{\includegraphics[scale=0.65]{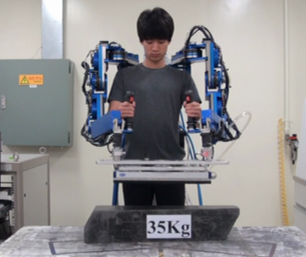}} 
\caption{Actual test of the robot being used to lift and maneuver a payload (weight is not known to the controller). (d)-(m) Photographs of the robot in use. Unknown environmental interaction occurred at (e) and (l).}
\label{fig:exp_actual}. 
\end{figure*}

\subsubsection{Experiment 1 - Performance limit}

The control gain was tuned using Theorem \ref{thm:pm} similar to Section \ref{section4-1}. To this end, with $\bK_P = 20 \bI$, the value of $\epsilon$ was set to $0.004 \rightarrow 0.002 \rightarrow 0.001$ (corresponding $\bK$ was $250\bI \rightarrow 500 \bI \rightarrow 1000\bI$). The human operator applied a sinusoidal input force in the $x$-direction. The values of $|e_{nr}[1]|_\infty$ were $0.046$, $0.022$, and $0.011$ for $\epsilon=0.004$, $0.002$, and $0.001$, respectively (the resulting values are provided without plots due to the space limit). As a result, $\bK_P = 20 \bI$ and $\bK=1000\bI$ were chosen in the following experiments.

\subsubsection{Experiment 2 -  Wall sticking effect}

Similar to Section \ref{sec:wall_sticking}, this section verifies the interaction capability. 
\begin{enumerate}[(a)]
\item{For $t=5$ to $25$ $\mathrm{sec}$, the human operator operated the robot sinusoidally in the $x$-, $y$-, $z$-directions simultaneously. The robot followed the human operator's command (Fig. \ref{fig:exp_passivity}).}
\item{For $t=25$ to $40$ $\mathrm{sec}$, a force of approximately 20 $\mathrm{N}$ was applied in the negative $z$-direction. The robot could not move in response because of the contact with the environment (see $p_z$ in Fig. \ref{fig:exp_passivity}). The nominal position (or velocity) did not  increase even though the force was constantly applied (see $p_{n,z}$ and $f_{h,z}$ in Fig. \ref{fig:exp_passivity}).}
\item{For $t=40$ to $50$ $\mathrm{sec}$, subsequent to the interaction with the environment, sinusoidal input forces in the $x$-, $y$-, and $z$- directions were again simultaneously applied. Note that the robot did not suffer from wall sticking effect. }
\end{enumerate}

\subsubsection{Experiment 3 - Payload carrying}

The robot was used to lift a payload with a mass of 35 $\mathrm{kg}$, which was unknown to the controller. Fig. \ref{fig:exp_actual}a-c shows the position and forces measured during the operation. As shown in the figure, the payload was lifted in the $z$-direction, moved sinusoidally in the $y$-direction, and then placed on the ground. Fig. \ref{fig:exp_actual}d-m shows photographs of the experiment. This experiment verified that the proposed admittance controller can be applied to  payload carrying applications with unknown environmental interaction.

\section{Conclusion}
\label{sec:conclusion}

This study proposed a passivity-based nonlinear admittance controller that fully captures the nonlinearities of robot dynamics. The proposed control structure can be represented as a feedback interconnection of passive subsystems, and therefore dissipativity of the closed-loop dynamics is guaranteed. By virtue of passivity, the robot can interact with the human operator and unknown environment stably. The two-time scale analysis was used to show that the proposed control scheme is an admittance controller. Moreover, the performance was analyzed for the finite control gains. The analysis shows that the performance increases linearly with the control gain.

The proposed approach was verified using a 1 DoF testbed. Moreover, experiments were performed using an actual powered upper-limb exoskeleton robot. After validating the performance analysis and interaction capabilities, our original purpose of this study was validated. Namely, the powered upper-limb robot was used to lift and maneuver a payload of unknown mass under the unknown environmental interaction.


%

%
%
%
%
%
%

\ifCLASSOPTIONcaptionsoff
  \newpage
\fi



\bibliographystyle{IEEEtran}
\bibliography{IEEEabrv,TMECH_upper_limb}
\begin{IEEEbiography}[{\includegraphics[width=1in,height=1.25in,clip,keepaspectratio]{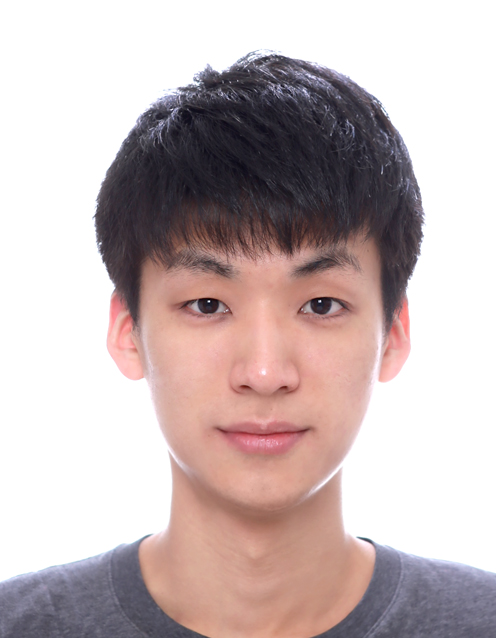}}]{Min Jun Kim}
	received the B.S. degree in mechanical engineering from Korea University, Seoul, Korea, in 2010, and received the Ph.D. degree in mechanical engineering from the Robotics Lab, Pohang University of Science \& Technology (POSTECH), Pohang, Korea, in 2016. He is currently a research scientist in Robotics and Mechatronics Center (RMC), German Aerospace Center (DLR). 

	His research interests include nonlinear robot control, flexible joint robots, kinematically redundant robots, physical human-robot interaction, and control of under-actuated systems.

\end{IEEEbiography}
\begin{IEEEbiography}[{\includegraphics[width=1in,height=1.25in,clip,keepaspectratio]{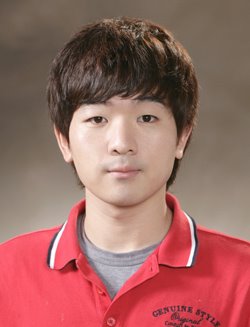}}]{Woongyong Lee} 
	 received the B.S. degree in mechanical engineering from Sungkyunkwan University, Seoul, Korea, in 2012. He is currently working toward the Ph.D. degree with the Robotics Laboratory, Pohang University of Science and Technology, Pohang, Korea.
	 
	 His research interests include the development and robust control of electro-hydraulic systems, and their applications to the robotic systems.
\end{IEEEbiography}
\begin{IEEEbiography}[{\includegraphics[width=1in,height=1.25in,clip,keepaspectratio]{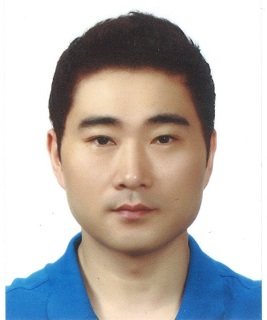}}]{Jae Yeon Choi} 
received the B.S., M.S., and Ph.D. degree from the Department of Control and Instrumentation Engineering, Hanyang University, Seoul, South Korea, in 2002, 2004, and 2012, respectively. He is currently a Senior Researcher of the Korea Institute of Robot and Convergence, Pohang, South Korea. 

His research interests include electro-hydraulic actuators, power assist robot, wearable robot, and kinematically redundant robot systems.
\end{IEEEbiography}
\begin{IEEEbiography}[{\includegraphics[width=1in,height=1.25in,clip,keepaspectratio]{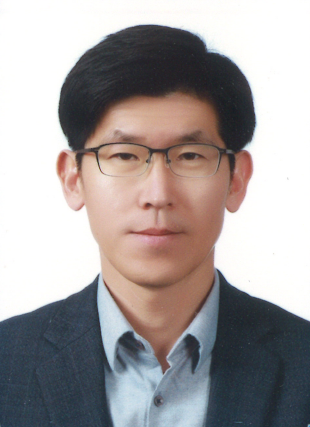}}]{Goobong Chung} 
received the B.S., M.S., and Ph.D. degree from the Department of Control and Instrumentation Engineering, Hanyang University, South Korea, in 1998, 2000, and 2005, respectively. From 2006 to 2007, he was a Postdoctoral Fellow with the Robotics Institute, Carnegie Mellon University, USA. In 2008, he joined the Korea Institute of Robot and Convergence, Pohang, South Korea, where he is currently a Senior Director of R\&D Division.

His research interests include the development of pipeline inspection robots, mobile manipulators, wearable robots, and an electro-hydraulic actuator for robots.
\end{IEEEbiography}
\begin{IEEEbiography}[{\includegraphics[width=1in,height=1.25in,clip,keepaspectratio]{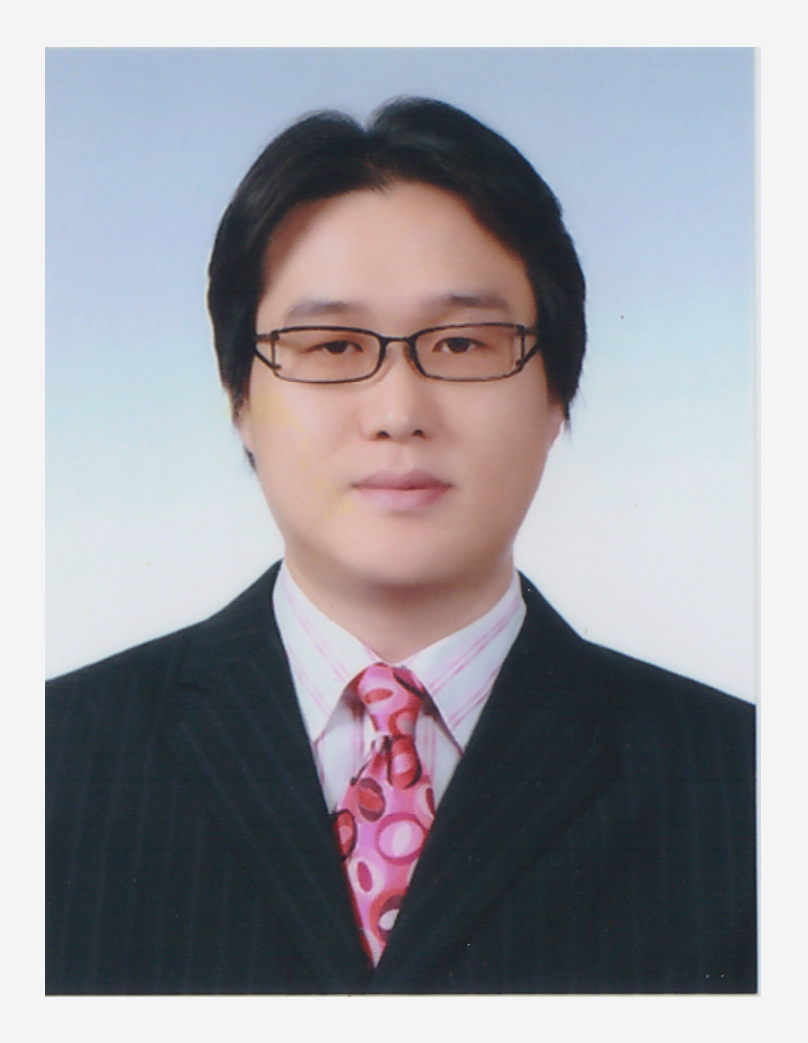}}]{Kyung-Lyong Han} 
received B.S. and M.S. degree in Electronic and Electrical Engineering from Pohang University of Science and Technology (POSTECH), Pohang, Korea, in 2002 and 2004, and received the Ph.D. degree in Electronic and Electrical Engineering from Robotics and Automation Lab., POSTECH, in 2011. He worked as a assistant research engineer at the Automation Research Team of the Production Technology Research Institute, Samsung Heavy Industries from January 2004 to March 2006, and principal research engineer at the Control and Instrumentation Research Group of the Technical Research Laboratories, POSCO from February 2011 to November 2018. He is a principal engineer in Artificial Intelligence Team of Mobile Communications Business, Samsung Electronics since December 2018.

His research interests include the development of social robot and industrial robot, robot control, and artificial intelligence algorithm.
\end{IEEEbiography}
\begin{IEEEbiography}[{\includegraphics[width=1in,height=1.25in,clip,keepaspectratio]{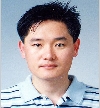}}]{Il Seop Choi} 
received the B.S. and M.S. degrees in electrical engineering from Pusan National University, Pusan, South Korea, in 1990 and 1992, respectively, and he received the Ph.D. degree in Automatic Control and Systems Engineering from the University of Sheffield, Sheffield, UK, in 2006. In 1992, he joined the Instrumentation and Control Research Group, POSCO, Pohang, South Korea, where he is currently a Senior Principal Researcher. 

His research interests include control, automation and robotics applications for steel manufacturing processes.
\end{IEEEbiography}
\begin{IEEEbiography}[{\includegraphics[width=1in,height=1.25in,clip,keepaspectratio]{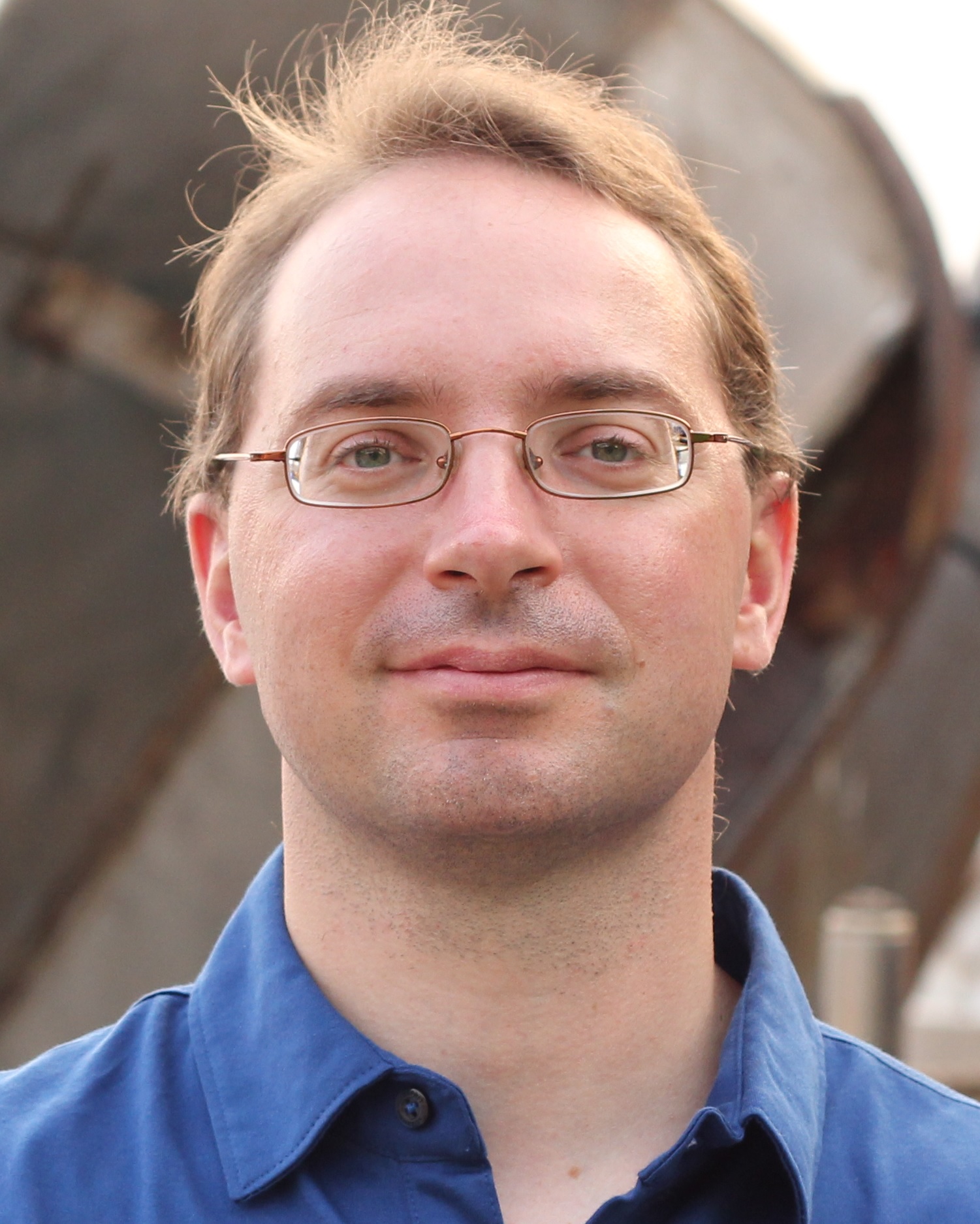}}]{Christian Ott} 
	received the Dipl.-Ing. Degree in mechatronics from the Johannes Kepler University (JKU), Linz, Austria, in 2001, and the Dr.-Ing. degree in control engineering from Saarland
	University, Saarbruecken, Germany, in 2005. From 2001 to 2007, he was with the German Aerospace Center (DLR), Institute of Robotics and Mechatronics, Wessling, Germany. From May 2007 to
	June 2009, he was a Project Assistant Professor in the Department of Mechano-Informatics, University of Tokyo, Japan. From 2011-2016 he has been working at DLR as team leader of the Helmholtz Young Investigators Group for ``Dynamic Control of Legged Humanoid Robots''. In January 2014, he became head of the Department of Analysis and Control of Advanced Robotic Systems at DLR. He received several scientific awards including an ERC consolidator grant, the ``Conference Best Paper Award'' at IEEE/RAS HUMANOIDS 2011, the “Dr.-Eduard-Martin”-Prize 2007, the Industrial Robot Outstanding Paper Award 2007, ICRA Best Video Award 2007, and the Best Paper Awards at the VDI-Conference ``Mechatronik-2005'' and from ``at-Automatisierungstechnik'' in 2005.

	His current research interests include nonlinear robot control, flexible joint robots, impedance control, and control of humanoid robots.
\end{IEEEbiography}
\begin{IEEEbiography}[{\includegraphics[width=1in,height=1.25in,clip,keepaspectratio]{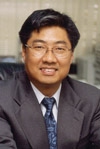}}]{Wan Kyun Chung} 
	received the B.S. degree in mechanical design from Seoul National University, Seoul, South Korea, in 1981, and the M.S. degree in mechanical engineering and the Ph.D. degree in production engineering from the Korea Advanced Institute of Science and Technology, South Korea, in 1983 and 1987, respectively. In 1987, he joined the School of Mechanical Engineering, Pohang University of Science and Technology, Pohang, South Korea, where he is currently a Professor. In 1988, he was a Visiting Professor with the Robotics Institute,  Carnegie Mellon University, Pittsburgh, PA, USA. In 1995, he was a Visiting Scholar with the University of California, Berkeley, CA, USA. 

	His research interests include the development of robust controllers for precision motion control, biomedical robotics and AI based biorobotics. Dr. Chung is an Editor-in-Chief of the Journal of Intelligent Service Robotics by Springer.
\end{IEEEbiography}

\end{document}